\documentclass[a4paper]{article}

\usepackage[utf8]{inputenc}
\usepackage{amsmath,amsthm, amssymb, amsfonts, bm, mathrsfs}

\usepackage{fullpage}

\usepackage[a4paper, margin=1in]{geometry}
\usepackage[pdfencoding=auto]{hyperref}

\usepackage{enumerate}
\usepackage{nicefrac}

\usepackage{xcolor}
\usepackage{subfig}
\usepackage{graphicx}   

\newcommand{\Rd}{\mathbb{R}^d}
\def\vv{{\bm{v}}}

\newcommand{\cM}{\mathcal{M}}

\newcommand{\cR}{\mathscr{R}}
\newcommand{\Rdn}[1]{\mathscr{R}\{#1\}}
\newcommand{\Rdns}[1]{\mathscr{R}^*\{#1\}}
\newcommand{\Sc}{\mathscr{S}}

\newcommand{\Rnorm}[1]{\left\|#1\right\|_{\mathscr{R}}}

\newcommand{\R}{\mathbb{R}}

\renewcommand{\S}{\mathbb{S}}

\newcommand{\SR}{\S^{d-1}\times \R}

\newcommand{\relu}[1]{\left[#1\right]_+} 

\newcommand{\cf}{\emph{cf.}~}
\newcommand{\ie}{\emph{i.e.},~}
\newcommand{\eg}{\emph{e.g.},~}

\newcommand{\Rbar}[1]{\overline{R}(#1)}
\newcommand{\Rbarone}[1]{\overline{R}_1(#1)}
\DeclareMathOperator{\Lip}{Lip}

\newcommand{\norm}[1]{\left\lVert{#1}\right\rVert}
\newcommand{\abs}[1]{\left\lvert{#1}\right\rvert}
\newcommand{\scal}[2]{\langle #1,#2\rangle}

\newcommand{\oR}{\overline{R}}
\newcommand*\diff{\mathop{}\!\mathrm{d}}

\newcommand{\intbias}{c}
\def\relu#1{\left[{#1}\right]_+}

\newcommand{\half}{\frac{1}{2}}
\newcommand{\sign}{\operatorname{sign}}
\def\prn#1{\left({#1}\right)} 

\newtheorem{definition}{Definition}[section]
\newtheorem{theorem}[definition]{Theorem}
\newtheorem{proposition}[definition]{Proposition}
\newtheorem{lemma}[definition]{Lemma}

\newtheorem{remark}[definition]{Remark}

\numberwithin{equation}{section}

\title{Integral representations of shallow neural network
with Rectified Power Unit activation function}

\author{Ahmed Abdeljawad%
\thanks{Johann Radon Institute,
Altenberger Straße 69,
A-4040 Linz,
Austria}
\, and Philipp Grohs\footnotemark[1]\,
\thanks{Faculty of Mathematics,
University of Vienna,
Oskar-Morgenstern-Platz~1,
A-1090 Vienna, Austria}
\thanks{Research Platform Data Science @ Uni Vienna,
Währinger Straße 29/S6,
A-1090 Vienna, Austria}
\\
ahmed.abdeljawad@ricam.oeaw.ac.at,\; philipp.grohs@univie.ac.at}

\date{\vspace{-5ex}}

\begin{document}

\maketitle

\begin{abstract}
In this effort, we derive a formula for the integral representation
of a shallow neural network with the Rectified Power Unit activation function. 
Mainly, our first result deals with the univariate case of representation capability of RePU shallow networks.
The multidimensional result in this paper  
characterizes the set of function that can be
represented with bounded norm and possibly unbounded width.

\end{abstract}


\section{Introduction}

A theoretical understanding of which functions can be well approximated by neural networks
has been studied extensively in the field of approximation theory with
neural networks (\eg \cite{abdeljawad21,boelcskei19,caragea20, eldan16,gribonval19,
grohs21, hornik89, pinkus99, schmidt-hieber20,telgarsky16, yarotsky17}).
In particular, integral representation techniques for \emph{shallow neural networks}
have received increasing attention (\eg \cite{candes99, e20, ito91, kainen00, kainen10, korolev21,
long21, parhi21, pieper20}).
Indeed, many authors have focused on the complexity control and generalization
capability for shallow neural networks.
The main motivation which makes this direction of research active is that
"the norm of the weights is more important than the depth of the network"
\eg \cite{bartlett97, shpigle19, wei012, zhang16}.

A shallow neural network, with $k$ neurons can be defined
as a real valued function defined on $\Rd$ of the following form:
$$
\sum_{i=1}^k a_i \sigma(\scal{w_i}{x} -b_i) +c
$$
where $\sigma$ is a non-linear function and $k$ is called the width of the network.
Moreover, $w_i\in \Rd$ and $b_i \in \R$ are the inner weights and
$a_i \in \R$ are the outer weights, and $c\in \R$.
The non-linear
function which acts componentwise is called activation function.
Here we use non-decreasing
homogeneous activation functions such as
RePU activation function defined as
$$
\sigma(x) =\max(0, x)^p :=[x]_+^p
$$
for $p\in \mathbb{N}$. If $p=1$ then $\sigma$ is the
well known \emph{Rectified Linear Unit} (ReLU), with
has seen considerable recent empirical success.
Moreover, if $p=2$, then $\sigma$ is the
\emph{Rectified Quadratic Unit} (ReQU), which starts to get the attention
of researchers thanks to its representation capability.
RePU activation functions have many advantages
such as homogeneity and regularity.
Moreover, the representation capability of a RePU neural network makes it
an interesting activation function, for instance, RePU can represent the identity
on bounded domains for all $p\in \mathbb{N}$,
the square and hence the product on bounded domains for any integer $p\geq 2$.
Note that our analysis can be applied to a larger
class of activation functions satisfies similar properties
to RePU activation function.

In several papers, the study of deep learning shows that
the size of the network tends to infinity when the approximation
 error goes to zero.
In our paper, we analyze approximation while controlling the norm of the weights,
instead of the size of the networks.
Hence we can consider networks with unbounded
or infinite size. Thus our focus relies on the understanding of
the \emph{representation cost} which is determined by the minimal norm
required to represent a function by shallow neural networks with
an architecture of unbounded size.

\par

In this paper, we develop the characterization of the  size of weights
required  to realize a given function by an unbounded width  single hidden-layer
with RePU activation function, in the univariate and multivariate cases.
For simplicity, we assume that the target function is real valued,  but we can easily
generalize our results to the multi-dimensional output case.

\subsection{Contributions}\label{subsec:contibution}
Mainly, our contribution is motivated by the problem of
understanding approximation capability of shallow neural networks.
Hence, our approach uses similar techniques as in the papers
\cite{savarese2019infinite} and \cite{Ongie2019}.
Our contributions can be summarized as follows:

\par

First, we show that  for a function $f$  that can be presented as a
shallow neural network $\sum_{i=1}^k a_i[\scal{w_i}{x} - b_i]_+^p + c$,
such that $p\geq 1$,
the minimal cost $\inf \nicefrac 12 \sum_{i=1}^k\left( \|w_i \|_2^2 + |a_i|^{\nicefrac 2p}\right)$
( where $w_i\in \R^d$ and $a_i\in \R$ )
is the same as $\inf \nicefrac 12 \sum_{i=1}^k  |a_i|_{\nicefrac 1p}^{\nicefrac 1p}$
( such that $w_i\in \mathbb{S}^{d-1}$ and $a_i\in\R$)
for any input dimension $d\in  \mathbb{N}$.
It is worth to mention that this is a well known result for 
networks with ReLU activation function.

In Theorem \ref{thm:main_1d}, we deal with real valued functions defined on $\R$,
that is univariate functions with single-output.
We show that the minimal representation cost $\Rbar{f}$ (see \eqref{eq:opt0})
of a  univariate function $f$ to
be represented as infinite width single-hidden layer neural network
with RePU activation function
is the same as the following quantity:
$$
 \max \left( \, \frac{1}{p!}\int_{-\infty}^\infty \abs{f^{(p+1)}(b)} \diff b \,
, \,\frac{1}{p!} \abs{f^{(p)}(-\infty) + f^{(p)}(\infty)} \, \right)^{\nicefrac 1p}.
$$

In the multivariate case, we use tools and techniques
from harmonic analysis and measure theory  toward
the main result.
Mainly, we characterize the representational cost for any real valued function
defined on $\R^d$ that can be represented by RePU neural networks
such that $p $ and $d$  are odd integers. Mainly the cost is given by the following quantity
$$
\Rnorm{f}^p = \norm{\frac{\gamma_d}{p!}\Rdn{\Delta^{(d+p)/2}f}}_{\mathcal{M}^1}
	 	 =\norm{ \frac{\gamma_d}{p!}\partial_b^{d+p}\Rdn{f}}_{\mathcal{M}^1} 
$$
where $\mathscr{R}$ is the Radon transform,
we refer to Section \ref{sec:radonM_T} and Section \ref{sec:rnorm} for notations and  details.
Moreover, it is worth to mention that our characterization is valid only for
odd dimension and odd power in the RePU activation function.
We provide a comprehensive view of the the representational cost $\Rnorm{f}$
of a function $f$, which is in particular finite if and only if $f$ well approximated by
a bounded norm RePU network, \ie have finite $\mathscr{R}$-norm.
It turns out that using RePU activation function leads to the
appearance of unregularized monomial units which needed a correction.
Instead in the multivariate cases in \cite{Ongie2019},
authors only have dealt with a single unregularized linear unit.
Although our results might seem straightforward
at first glance, actually proving them  is quite delicate.

\subsection{Related work}\label{subsec:related_work}

In real life applications, shallow neural networks are less popular
than deep neural networks, although theoretically shallow networks are better understood.
For instance, integral representation of neural networks have been established to explore
the expressive power of shallow neural networks \eg
\cite{barron93, murata96}, and to estimate the approximation errors \eg \cite{kurkova12}. 

\par 

In \cite{barron93}, Barron showed that a single hidden-layer neural network
with sigmoidal activation function can approximated 
any continuous function on a bounded domain up to  any given precision,
where the approximation error related to the number of nodes in the network.
Instead, in \cite{candes99}, Cand\`es used single-hidden layer neural network
with oscillatory activation function  and methods of harmonic analysis to
the problem of representing a function in terms of shallow neural networks.

Ito \cite{ito91} established a uniform approximation on the whole space $\R^d$,
where he used a shallow neural network with step or sigmoid activation function.
In the later paper, the inversion formula of the Radon transform plays a crucial role,
during the proof of the integral representation. Moreover, Ito used a
function as  outer weights depends on the inner weights, which also satisfies
regularity and growth assumptions.

K\r{u}rkov\'a, in \cite{kurkova12},  showed estimates of network complexity, 
in representing a multivariate function
through shallow neural networks such that the author considered
shallow networks as integral transforms with kernels corresponding to network units.
Since the classes of functions that can be expressed 
as integral with kernels are sufficiently large,
many authors contributed to characterize and determine those classes,
for instance,  all sufficiently smooth compactly supported functions
or functions decreasing sufficiently rapidly at infinity
can be expressed as networks with infinitely many Heaviside perceptrons
\cf \cite{kainen07, kainen10, kukova97}.

Our work is also related to \cite{savarese2019infinite, Ongie2019}.
In \cite{savarese2019infinite}, authors  initiated the study of the representational cost
of univariate functions
in term of weight magnitude for shallow neural networks with
ReLU activation function. 
Ongie et al.,  have been extended this approach to understand what kind of multivariate functions
can be represented by infinite-width shallow neural networks with ReLU activation function,
where the key analysis tool used in their main problem is the Radon transform
\cf  \cite{Ongie2019}. 

In subsequent papers, shallow networks  have been studied from various perspectives
for different objectives.
For example,  in \cite{ji20}, generic scheme have been introduced to approximate
functions with the so-called neural tangent kernel
by sampling from transport mappings between the initial weights and
their desired values. Instead, in \cite{chizat20},
authors have been analyzed the training and generalization behavior
of infinitely wide two-layer
neural networks with homogeneous activation. From optimization point of view,
Pilanci and Ergen 
showed \cite{pilanci20} the equivalence between  $\ell^2$-regularized empirical risk minimization
of shallow ReLU networks and the finite-dimensional convex group-lasso problem.

Lastly, we remark that the theoretical analysis  and applications of shallow neural networks
is growing in terms of number of published papers and researchers contributing
to this field.

\subsection{Organization of the Paper}\label{subsec:organization}

\par

The organization of this paper is as follows. In Section \ref{sec:infnets},
we describe the problem setting and define notions of
shallow neural networks with RePU activation function.
In section \ref{sec:radonM_T}, we recall
the required preliminary concepts  and definitions
of Radon measure and Radon transform and its dual.
In Section \ref{sec:main_univ}, based on the definitions in Section \ref{sec:infnets},
we state our main result  for the case of univariate functions.
Finally, in Section \ref{sec:rnorm},  using definitions in Section \ref{sec:infnets}
and the Radon transform from Section \ref{sec:radonM_T},
we characterization the norm of weights
required to realize a multivariate function as  an unbounded single hidden-layer
with RePU activation function.

\section{Infinite width RePU neural networks}\label{sec:infnets}

Infinite width representations with ReLU activation function
have been considered in \eg\cite{Ongie2019, Petrosyan2020, savarese2019infinite}.
In the current section, we introduce shallow neural networks
with RePU activation function in the infinite width setting.

For fixed $p\in \mathbb{N}$, we can write two layers network with an unbounded number of neurons,
$d$ dimensional input and one-dimensional output,
as a function $g_{\theta, p}: \mathbb{R}^d \rightarrow \R$ given by
\begin{equation}\label{eq:finitenet}
g_{\theta, p}(x) = \sum_{i=1}^k a_i[\scal{w_i}{x} - b_i]_+^p + c,~~\text{for all}~~x\in\mathbb{R}^d
\end{equation}
$k$ represents the width where $k \in \mathbb{N}$, $w_i$ are the rows of $W\in \mathbb{R}^{k \times d}$
and the parameters $\theta = (k,W = (w_1,...,w_k),b = (b_1,...,b_k),a = (a_1,...,a_k),c)\in \Gamma $
such that
$$
\Gamma:=\left\{\theta=\left(k, W, b, a,c \right) \mid
k \in \mathbb{N}, W \in \mathbb{R}^{k \times d}, {b} \in \mathbb{R}^{k},
a \in \mathbb{R}^{k}, c \in \mathbb{R}\right\}.
$$

In the previous definition of the network function $g_{\theta, p} (x)$ we can exclude the bias $c$,
since it can be simulated through an additional unit,
for instance
$b_{k+1}=-1$, $a_{k+1}=c$, and $w_{k+1}= 0$.
Instead, the unregularized biases $\{b_i\}_{i=1}^k$ in the hidden layer are crucial to our analysis,
thus removing or regularizing them will lead to substantial changes.

For any $\theta\in\Gamma$ we define a cost function $C(\theta)$ by
\begin{equation}
C(\theta) =\frac{1}{2}\left( \|W\|_F^2 + \|a\|_{\nicefrac{2}{p}}^{\nicefrac{p}{2}} \right)
= \frac{1}{2}\sum_{i=1}^k\left( \|w_i\|_2^2 + |a_i|^{\nicefrac{2}{p}} \right),
\label{eq:c}
\end{equation}
where $\|W\|_F$ is the Frobenius norm of the matrix $W$.
Based on the previous definition of the cost, the minimal representation cost necessary
to represent a given function $f\in \R^d\rightarrow \R$ in terms of a shallow neural network is given by
\begin{equation}\label{eq:Rf}
R(f) := \inf_{\theta \in\Gamma} C(\theta)\text{ such that }f = g_{\theta, p}.
\end{equation}
Thanks to the $p$-homogeneity of the RePU activation function,
we can show (see Appendix \ref{app:neyshaburreproof})
that minimizing $C(\theta)$
is the same as restricting the inner layer weight vectors $\{w_i\}_{i=1}^k$
to be the unit norm while minimizing the $\ell^{\nicefrac{1}{p}}$ quasinorm
of the outer layer weight $ a$, more details can be found in
Appendix \ref{app:neyshaburreproof}.

Let $\Theta$ be the collection of all $\theta \in \Gamma$
such that, for any $i\in \{1, \dots, k\}$, $w_i$ belongs to the unit sphere
$\mathbb{S}^{d-1} := \{w\in \mathbb{R}^d : \|w\| = 1\}$, hence we have
\begin{equation}\label{eq:plainr}
R(f) = \inf_{\theta \in \Theta}
\|a\|_{\nicefrac{1}{p}}^{\nicefrac{1}{p}}\text{ such that }f = g_{\theta, p}.
\end{equation}
In the case where $p=1$, $R(f)$ is finite if $f$ is a continuous
piecewise linear function with finitely many pieces that can be represented by
a finite width two layers ReLU networks. Similarly, when $p=2$,
$R(f)<\infty$ if $f$ is realizable as a finite width
two layers ReQU networks, that is $f$ is a continuous
piecewise quadratic function with finitely many pieces.
For general $p$, $R(f)$ is finite only if $f$ is a continuous piecewise
polynomial of order $p$ with finitely many pieces, which can be represented by
a finite width two layers RePU neural networks.
\par

Since we focus on the minimal norm required to represent a function with an infinite-size network,
we use a modified representational cost, rather than \eqref{eq:plainr}, which
captures larger space of functions.
The following defines the minimal limiting representational cost of all
sequences of shallow neural networks
converging to $f$ uniformly,

\begin{equation}\label{eq:rbar0}
  \Rbar{f} := \lim_{\varepsilon\rightarrow 0}\left(\inf_{\theta \in \Theta}
   C(\theta)~~s.t.~~\abs{\left(g_{\theta, p} - f\right)(x)}\leq \varepsilon
   \text{ for any } \norm{x}\leq \frac 1 \varepsilon
   ~\text{ and }g_{\theta, p}(0)=f(0)\right).
\end{equation}

Consequently, if $\Rbar{f}$ is finite, then $f$ can be represented
by an infinite width shallow RePU network,
where the outer-most weights are described by a density
$\mu(w,b)$ for all weights and bias pairs $(w,b) \in {\mathbb{S}^{d-1}\times \R}$.
We denote by $ {\mathcal{M}^1}({\mathbb{S}^{d-1}\times \R})$ the space of signed measures
$\mu$ defined on $(w,b)\in{\mathbb{S}^{d-1}\times \R}$  with finite total variation norm
$
\|\mu\|_ {\mathcal{M}^1} = \int_{{\mathbb{S}^{d-1}\times \R}} d|\mu|
$, more details can be found in Section \ref{sec:radonM_T}.
Let  $c\in\R$, and $p\in \mathbb{N}$,
we define the infinite width two layers RePU network $H_{\mu,c}^p$
as follows:
\begin{equation}\label{eq:infinitenet0}
	\mathtt{H}_{\mu, c}^p (x) := \int_{\mathbb{S}^{d-1}\times \R}
	\frac{1}{1+\abs{b}^{p-1}}\left([\scal{w}{x}-b]_+^p - [-b]_+^p \right)d\mu(w,b)  + c.
\end{equation}
In case $c=0$, we write $\mathtt{H}_{\mu}^p$ instead of $\mathtt{H}_{\mu, 0}^p$.
Our definition is a correction of the one proposed firstly by \cite{savarese2019infinite}.
Moreover it is a generalization of the definition given in \cite{Ongie2019}.
Indeed when $p=1$ then $2\mathtt{H}_{\mu, c}^1$ is the same shallow network integral representation
given in \cite[Equation (8)]{Ongie2019}.
Moreover, our definition ensures that the integral is well defined, since the integrand
$\frac{1}{1+\abs{b}^{p-1}}\left([\scal{w}{x}-b]_+^p - [-b]_+^p \right) $ is continuous and bounded
for any $(w, b)\in \mathbb{S}^{d-1}\times \R$ and fixed $x\in \R$.

Let $\psi(b)=1+|b|^{p-1}$, then we show that $\Rbar{f}$ in  \eqref{eq:rbar0},
is the same as the following: 
\begin{equation}\label{eq:opt0}
	\Rbar{f} = \min_{\mu\in \mathcal{M}^1({\mathbb{S}^{d-1}\times \R}),c\in \R}
	\big\|{\mu}\big\|_{\mathcal{M}^1(1/\psi)}^{\nicefrac{1}{p}}
	~~\text{ such that }~~f = \mathtt{H}_{\mu,c}^p,
\end{equation}
the  norm notation is defined in Section \ref{sec:radonM_T}
while the details of the proof are in Appendix \ref{appsec:opt}.

\par

\begin{remark}
There is an equivalence between the loss minimization approach
by controlling $C(\theta)$ while fitting some loss function $L(g_{\theta, p})$
to train an infinite width RePU  neural network $g_{\theta, p}$
and learning a function $f$ by controlling $\Rbar{f}$ while fitting the loss $L(f)$.
More specifically, for any hyper parameter $\lambda\in \mathbb{R}$,
we have the following
\begin{equation}\label{eq:equiv_C_R}
	\min_{\theta \in \Theta} L(g_{\theta, p}) + \lambda C(\theta)
	\iff
	\min_{f:\mathbb{R}^d\rightarrow \R} L(f)+\lambda \Rbar{f}.
\end{equation}
The right-hand side of \eqref{eq:equiv_C_R} gives information about the class of
functions we are minimizing with infinite width,
moreover it describes what we are minimizing with a finite,
but sufficiently large, width.
\end{remark}

\par

\section{Radon measure and Radon transform}\label{sec:radonM_T}

\par
In this section we recall some definitions and properties of \emph{Radon measure} and 
\emph{Radon transform}. Both ingredients play an important role in our analysis.

\subsection{Radon measure}\label{appsec:infnets}
Let $\mathcal{B}$ be the $\sigma$-algebra of all Borel sets.
A Radon measure is a positive Borel measure
$\mu : \mathcal{B} \rightarrow [0, +\infty]$
which is finite on compact sets and is inner regular in the sense that
for every Borel set $E$ we have
$$
\mu(E) = \sup\{\mu({K}) : K \subset E, K\in \mathcal{K}\}
$$
$\mathcal{K}$ denoting the family of all compact sets.
For a signed measure 
\footnote{We assume $\mu$ is a signed \emph{Radon} measure;
see, e.g., \cite{malliavin2012} for a formal definition.
In this paper  a \emph{measure} means \emph{Radon measure} 
to avoid confusion with the Radon transform.
}  $\mu$
defined on $\SR$, $\|\mu\|_{\cM^1} = \int d|\mu|$ denotes its total variation norm.
Let $\cM^1(\SR)$ denotes the space of measures on $\SR$
with finite total variation norm.
Moreover, for  a measure $\mu$ and a positive function $\omega$ defined on $\SR$
such that $\omega$ is integrable with respect to $\mu$,
we define a ``weighted'' version of the total variation norm
as follows
$$
\|\mu\|_{\cM^1(\omega)} := \| \omega \mu  \|_{\cM^1} := \int \omega d|\mu|.
$$
Since $\SR$ is a locally compact space, $\cM^1(\SR)$ is the Banach space dual
of $C_0(\SR)$, the space of continuous functions on $\SR$ vanishing at infinity
\cite[Chapter 2, Theorem 6.6]{malliavin2012}, and
\begin{equation}\label{eq:tvnorm}
    \|\mu\|_{\cM^1} =
    \sup \left\{\int \varphi\, d\mu : \varphi\in C_0(\SR), \|\varphi\|_\infty \leq 1 \right\}.
\end{equation}
The weighted version of the total variational norm is defined in a similar way as \eqref{eq:tvnorm}.
That is, 
$$
 	\|\omega \mu\|_{\cM^1} =
    \sup \left\{\int \varphi\,\omega  d\mu : \varphi\in C_0(\SR), \|\varphi\|_\infty \leq 1 \right\},
$$
if the weight $\omega$ has nice properties then we can relax the assumptions on $\varphi$.
In our paper $\omega(w, b)=\omega(b)= \frac 1{1+|b|^{p-1}}\in C_0(\SR)$, hence
it is enough to consider the continuity assumption only on $\varphi$.
We write $\langle \mu, \varphi \rangle$ to denote $\int \varphi d\mu$, 
where $\mu\in \cM^1(\SR)$ and $\varphi\in C_0(\SR)$.

Let $C_b(\SR)$ denotes the space of continuous and bounded functions
on $\SR$.  Hence any measure $\mu \in \cM^1(\SR)$ can be extended uniquely to a continuous
linear functional on $C_b(\SR)$.
Therefore,  for all $x\in\Rd$,  the infinite width network
\begin{equation}
	\mathtt{H}_\mu^p(x) := \int_{\SR}\prn{\frac{[ w x - b]_+^p-[-b]_+^p}{1+|b|^{p-1}}} d\mu( w,b)
\end{equation}
is well-defined, since
$\varphi(w,b) =\frac{[ w x - b]_+^p-[-b]_+^p}{1+|b|^{p-1}}$
belongs to $C_b(\SR)$.

In our paper a measure $\mu \in \cM^1(\SR)$ is called \emph{even} measure if 
\begin{equation*}
  \int_{\SR} \varphi(w,b) d\mu(w,b) =  \int_{\SR} \varphi(- w,-b) d\mu( w,b)
  ~~\text{for all}~~\varphi \in C_0(\SR)
\end{equation*}
or $\mu$ is \emph{odd} measure if
\begin{equation*}
     \int_{\SR} \varphi( w,b) d\mu( w,b) =  -\int_{\SR} \varphi(- w,-b) d\mu( w,b)
     ~~\text{for all}~~\varphi \in C_0(\SR).
 \end{equation*}
Every measure $\mu \in \cM^1(\SR)$ is uniquely decomposable
as $\mu = \mu^+ + \mu^-$ where $\mu^+$ is even and $\mu^-$ is odd,
which called the even and odd decomposition of $\mu$.
For instance, let $\mu$ be a measure with a density $g(w,b)$
then $\mu^+$ is the measure with density $g^+(w,b) = \frac{1}{2}(g(w,b) + g(-w,-b))$
and $\mu^-$ is the measure with density $g^-(w,b) = \frac{1}{2}(g( w,b) - g(- w,-b))$.

We denote by $\cM_{e}^1(\SR)$ the subspace of all even measures in $\cM^1(\SR)$.
The space $\cM_{e}^1(\SR)$ is the Banach space dual of $C_{0, e}(\SR)$,
the subspace of all even functions $\varphi \in C_0(\SR)$.  
Even measures play an important role in our paper, mainly our main results
use even measure in the characterization of the representational cost
\cf \eg Lemma \ref{lem:uniqueness}.
The following result can be found in \cite{Ongie2019}.
\begin{lemma}\label{prop:evenodd}
Let $\mu \in \cM^1(\SR)$ with $\mu = \mu^+ + \mu^-$ where
$\mu^+$ is even and $\mu^-$ is odd. 
Then $\|\mu^+\|_{\cM^1} \leq \|\mu\|_{\cM^1}$
and $\|\mu^-\|_{\cM^1} \leq \|\mu\|_{\cM^1}$.
\end{lemma}

\subsection{The Radon transform and its dual}\label{sec:radon}

\par

 \emph{Radon transform} is a fundamental transform applicable to tomography,
\ie the creation of an image from the scattering data associated
with cross-sectional scans of an object. 
In this section, we recall some aspects and features of
the Radon transform mentioned in 
Helgason's book, more details about its properties and applications can
be found in  \cite{helgason1999radon}.

\par

In our approach the representational cost $\Rbarone{f}$ in Section \ref{sec:rnorm} is characterized
in terms of the Radon transform.
Therefore, we briefly  review the Radon transform and its dual.
Note that several properties of the Radon transform are analogous
to the properties of the Fourier transform.

\par
Let $f$ be a real-valued function defined on $\Rd$ and integrable on each hyperplane in $\Rd$.
The Radon transform $\cR$ represents  $f$ in terms of its integrals over all possible hyperplanes in $\Rd$,
in the following way

\begin{equation}
	\Rdn{f}(w,b) := \int_{\scal{w}{x} = b} f(x)\,dx~~\text{for all}~~( w,b)\in\SR,
	\label{eq:radon}
\end{equation}
where $dx$ represents integration with respect to the $d-1$ dimensional Lebesgue measure
on the hyperplane $\scal{w}{x}= b$.
In view of the fact that  $ \scal{w}{x} =b$ and $-\scal{w}{x} = -b$ determine
the same hyperplane, then the Radon transform
is an even function, that is, $\Rdn{f}( w,b) = \Rdn{f}(- w,-b)$ for all $(w,b)\in\SR$.

\par

The \emph{dual Radon transform} $\cR^*$, \ie the formal adjoint of $\cR$, of  a function
$\varphi: \SR \rightarrow \R$ is  defined as follows:
%
\begin{equation}
\Rdns{\varphi}( x ) := \int_{ \S^{d-1} } \varphi ( w , \scal{w}{x} )\,d w~~\text{for all}~~x \in \Rd,
\label{eq:adjoint}
\end{equation}
where $dw$ represents integration with respect to the surface measure of the unit sphere $\S^{d-1}$. 

The function $f$ can be recovered from the Radon transform employing the \emph{inverse Radon} formula.
The inverse Radon transform corresponds to the reconstruction of the function from the projections.
Mainly  the  inverse Radon transform is a composition of the dual Radon transform $\cR^*$
followed by a multiplier operator. 

The multiplier operator is given by $m_{d-1}(D) =(-\Delta)^{(d-1)/2}$,
such that for any $s>0$ the operator $m_s(D)=(-\Delta)^{s/2}$ is defined by
\begin{equation}\label{eq:multipllier}
m_s(D)f(x) = \mathscr{F}^{-1}\left( m(\xi)\left(\mathscr{F}f\right)(\xi)\right)(x)
\end{equation} 
where  $\mathscr F$ and  $\mathscr F^{-1}$ are the Fourier transform and the inverse Fourier transform
defined respectively by
\begin{align*}
(\mathscr Ff)(\xi )&= \widehat f(\xi ) \equiv (2\pi )^{-\frac d2}\int _{\Rd} e^{-i\scal  x\xi } f(x)\, dx\\
(\mathscr F^{-1}f)(x )& \equiv (2\pi )^{-\frac d2}\int _{\Rd} e^{i\scal  x\xi } f(\xi )\, d\xi
\end{align*}
when $f\in L^1(\Rd)$.
When $d = 2n+1$ for $n\in \mathbb{N}$, $(-\Delta)^{(d-1)/2}= (-1)^n \Delta^n$
which is the Laplacian applied $n$ times.
While if $d$ is even it is a pseudo-differential operator given by convolution with a singular kernel.
In view of the previous operators  the explicit inversion formula for the Radon transform
is given by
\begin{equation}\label{eq:inv_radon}
f = \gamma_d (-\Delta)^{(d-1)/2}\Rdns{\Rdn{f}},
\end{equation}
where $\gamma_d$
is a constant depending on the dimension $d$, the equality in \eqref{eq:inv_radon} holds for $f$ 
belonging to many common function spaces \eg $f$ is a rapidly decreasing functions on $\Rd$,
more details can be found in \eg \cite{helgason1999radon}. 

Although more restriction are imposed on the function space,
the inversion formula for the dual Radon transform \eqref{eq:adjoint}
can be given in a similar way to the inversion formula for the Radon transform.
Indeed, if  $\varphi$ is  an even function belongs to the Schwartz class $\mathscr{S}(\SR)$,
then
\begin{equation}\label{eq:invform}
	 \varphi = \gamma_d \Rdn{(-\Delta)^{(d-1)/2} \Rdns{\varphi}},
\end{equation}
where $\gamma_d = \frac{1}{2(2\pi)^{d-1}}$.
We denote the Fourier transform with respect to $b$ by $\mathscr{F}_b$.
We recall the \emph{Fourier slice theorem} for the Radon transform
\cf Helgason \cite{helgason1999radon}.
The $\mathscr{F}_b$ of the Radon transform  is given by
\begin{equation}\label{eq:partial_fourier_radon}
	\mathscr{F}_b\Rdn{f}(w,\tau) = \widehat{f}(\tau \cdot w)
\end{equation}
where  $f \in L^1(\Rd)$,  $\tau\in\R$ and $w\in\S^{d-1}$.
From the definition of a multiplier operator in \eqref{eq:multipllier}
and the later equality
\eqref{eq:partial_fourier_radon} we establish the \emph{intertwining property} of 
the Laplacian and the Radon transform:
assuming $f$ and $\Delta f$ are in $L^1(\Rd)$, we have
\begin{equation}
  \Rdn{\Delta f} = \partial_b^2 \Rdn{f}\; 
\end{equation}
where $\partial_b$ is the partial derivative with respect  to $b$.
More generally for any positive integer $s$, assuming that $f$ and $(-\Delta)^{s/2}f$
are in $L^1(\Rd)$ we have
\begin{equation}\label{eq:interwining_ppt}
\Rdn{(-\Delta)^{s/2} f} = {(-\partial^2_b)}^{s/2} \Rdn{f}
\end{equation}
where fractional powers of $-\partial^2_b$ can be defined similarly 
as the fractional powers of the Laplacian.

\par

\section{Infinite width bounded norm networks for univariate function}\label{sec:main_univ}

In the next result, we consider a possible model for an infinite width network on one dimension.
In particular we characterize $\oR(f)$ for univariate function $f$.
Moreover, in this section, we assume that the measure has a density that satisfies
certain decay properties.
Mainly, we treat an infinite width neural networks $\mathtt{H}_{\mu,c}^p$
with density $\mu(w, b)$ that decays as $\frac1{1+|b|^{p+2}}$
when $|b|\rightarrow \infty$, this assumption is needed to ensure convergence
of certain integrals in the proof of Theorem \ref{thm:main_1d}.

\par

\begin{theorem}\label{thm:main_1d}
Let $p \in \mathbb{N}$ such that $p $ is odd, $f$ is
a real-valued function defined on $\R$ and $\psi(b) = 1+|b|^{p-1}$ where $b\in \R$ ,
then we have
\begin{align*}
\min_{\mu,\, c\in \R}
\big\|{\mu}\big\|_{\mathcal{M}^1(1/\psi)}^{\nicefrac 1p}&
~~\text{ such that }~~ f = \mathtt{H}_{\mu,c}^p
\\
&= \max \left( \, \frac{1}{p!}\int_{-\infty}^\infty \abs{f^{(p+1)}(b)} \diff b \,
, \,\frac{1}{p!} \abs{f^{(p)}(-\infty) + f^{(p)}(\infty)} \, \right)^{\nicefrac 1p}
\end{align*}
where the minimum is taken over all measures $\mu \in \mathcal{M}^1({\mathbb{S}^{d-1}\times \R})$
with density that decays as $\frac1{1+|b|^{p+2}}$ when $|b|\rightarrow \infty$.
\end{theorem}
The derivatives of $f$ are interpreted in the weak sense, that is
we work with distributions, rather than only functions.
Hence integrals in the statement and its proof are well defined even if $f$ is not differentiable.
Noting that we use the same symbols to refer to measure and its associated density.
Moreover, we denote by
$f^{(p)}(\infty)=\lim_{x\to\infty}f^{(p)}(x)$ and $f^{(p)}(-\infty)=\lim_{x\to-\infty}f^{(p)}(x)$.
Observe that both limits exist if $\int_{-\infty}^\infty \!\!\abs{f^{(p+1)}(x)} \diff x$ is finite.
Furthermore, in our paper, we deal with more relaxed assumptions on the measure $\mu$,
since we don't ask $\mu$ to have a density in the Schwartz space.

\begin{proof}
In one dimension, $w \in \{\pm 1\}$,
(by transforming $(w,b) \mapsto (w,wb)$, which does not change $\oR(f)$,
in view of a linear change of variables argument in \eqref{eq:infinitenet0})
it is more convenient to reparametrize the RePU units as
$\relu{w(x-b)}^p$ instead of $\relu{{w}{x} - b}^p$.
Let $H(z)$ denotes the Heaviside function, whose distributional derivative is the Dirac distribution $\delta$.
Then for any representation $f=\mathtt{H}_{\mu,c}^p$, we have
\begin{align}\label{eq:conv}
f(x) &= \int_{\R} \left( \mu(1,b) \frac{[x -b]_+^p - [-b]_+^p}{1+\abs{b}^{p-1}}
+ \mu(-1,b)\frac{[b -x]_+^p - [b]_+^p}{1+\abs{b}^{p-1}} \right) \diff b +\intbias.
\end{align}

For $\epsilon>0$, we let $x$ be fixed in the ball of center $0$ and radius $\frac 1\epsilon$,
we differentiate $f$ in \eqref{eq:conv} $p+1$ times with respect to $x$, then we get
for $k \in \{1, \dots, p\}$
\begin{equation}\label{eq:kDeriv}
f^{(k)}(x) = \frac{p!}{(p-k)!}\int_{\R} \mu(1,b) \frac{[x-b]_+^{p-k}}{1+\abs{b}^{p-1}}
+ (-1)^k \mu(-1,b)\frac{[b-x]_+^{p-k}}{1+\abs{b}^{p-1}} \diff b,
\end{equation}
hence
\begin{align}
f^{(p)}(x) &= p!\int_{\R} \mu(1,b) \frac{H(x-b)}{1+\abs{b}^{p-1}}
- \mu(-1,b)\frac{H(b-x)}{1+\abs{b}^{p-1}} \diff b \label{eq:fprime}\\
f^{(p+1)}(x) &= \int_\R \left(p!\frac{ \mu(1,b) + \mu(-1,b)}{1+\abs{b}^{p-1}} \right) \delta_x(b) \diff b \notag
\\
&=p!\frac{ \mu(1,x) + \mu(-1,x)}{1+|x|^{p-1}} = p! \frac{\mu_+(x)}{\psi(x)}, \label{eq:key}
\end{align}
where we set $\mu_+(x):=\left({ \mu(1,x) + \mu(-1,x)}\right)$ and $\psi(x)= 1+ |x|^{p-1}$.

    The quantity  $p!\frac{\mu_+}{\psi}$ uniquely represents
    $f^{(p+1)}$. Moreover, the measure $\mu$
    that represents the function  $f$ is almost unique
    (up to shifting argument with respect to $b$). We denote by 
    $\mu_-(b)=\mu(1,b)-\mu(-1,b)$.
    Since $w\in\{-1, 1\}$,
    $\mu_+ $ and $\mu_-$ define $\mu$ as following
    \begin{equation}\label{eq:mu}
    \mu(w,b)=\half \left( \mu_+(b) -\nicefrac 1p \sum _{k=1}^{p}(-w)^{2k+1}\cdot\mu_-(b) \right)
    = \half(\mu_+(b)+w\cdot\mu_-(b)).
    \end{equation}

    Note that, for odd $p$, we can write
    \begin{equation}\label{eq:OddrepuIdentity}
        \relu{x}^p = \half(|x|^p +x^p), \quad\text{for any } x \in \mathbb{R}.
    \end{equation}
    In view of \eqref{eq:mu} and \eqref{eq:OddrepuIdentity}, we have the following:

    \begin{align}
    f(x) &=  \int_{\R} \left( \mu(1,b) \frac{[x -b]_+^p - [-b]_+^p}{1+\abs{b}^{p-1}}
	+ \mu(-1,b)\frac{[b -x]_+^p - [b]_+^p}{1+\abs{b}^{p-1}} \right) \diff b +\intbias\nonumber
	\\*
	&=\half \int_{\R}\frac{1}{\psi(b)} \Big( \mu_+(b) \left({|x -b|^p - |b|^p}\right)
	+ \mu_-(b)\left({(x-b)^p + b^p}\right) \Big) \diff b +\intbias\nonumber
	\\*
    &= \frac{1}{2\cdot p!}  \int_\R f^{(p+1)}(b) \left(|x -b|^p - |b|^p\right) \diff b\nonumber
    \\* 
    & \qquad\qquad+ \frac{1}{2}\sum_{k=1}^p\left(  (-1)^{p-k}\binom{p}{k}
				\int_{\R} \frac{\mu_-(b)}{\psi(b)}b^{p-k}\diff b\right) x^k  
    + \left( \int_{\R} b^p \frac{\mu_-(b)}{\psi(b)} \diff b + \intbias \right). \label{eq:affine}
    \end{align}
    From the previous equation \eqref{eq:affine}, $\mu_-$
	contributes the coefficients of a polynomial
    of order $p$ in $f=\mathtt{H}_{\mu,c}^p$. Since we can 
    adjust the constant term using the bias $\intbias$
    without changing $\oR(f)$, the only important issue with choosing $\mu_-$ 
	is to get the correct polynomial term.
  To get a clear idea about which polynomial correction we need, it is important to consider 
   \eqref{eq:fprime} to evaluate $f^{(p)}(-\infty)$ and $f^{(p)}(+\infty)$.
   Due to the fact that $\mu \in \mathcal{M}^1(\SR)$, it is clear that 
   $\int \abs{f^{(p+1)}} \diff x$  is finite.
    Therefore $f^{(p)}$ converges at $\pm\infty$,  and
    \begin{equation}
    \begin{aligned}
    f^{(p)}(-\infty) + f^{(p)}(+\infty) 
    &= \int_\R (0-\mu(-1,b)p!\frac {1}{1+|b|^{p-1}})\diff b +
    	 \int_\R (\mu(1,b)p!\frac {1}{1+|b|^{p-1}}-0)\diff b
    \\
    &=p! \int_\R \frac{\mu_-(b)}{\psi(b)} \diff b.
    \label{eq-const_alpham}
    \end{aligned}
    \end{equation}
    Moreover when $k \in \{1, \dots, p-1\}$, in view of \eqref{eq:kDeriv},
    we have
    $$
    f^{(k)}(0) = \frac{p!}{(p-k)!}\int_{\R} \mu(1,b) \frac{[-b]_+^{p-k}}{1+\abs{b}^{p-1}}
	+ (-1)^k \mu(-1,b)\frac{[b]_+^{p-k}}{1+\abs{b}^{p-1}}  \diff b.
    $$
    If $k$ is an even number, using \eqref{eq:OddrepuIdentity}, then
    \begin{align*}
    f^{(k)}(0) &= \frac{p!}{(p-k)!}\int_{\R} \mu(1,b) \frac{[-b]_+^{p-k}}{1+\abs{b}^{p-1}}
	+ \mu(-1,b)\frac{[b]_+^{p-k}}{1+\abs{b}^{p-1}}  \diff b
	\\
    & =  \frac{p!}{(p-k)!}\int_{\R} \mu(1,b) \frac{|b|^{p-k} - b^{p-k}}{1+\abs{b}^{p-1}}
	+ \mu(-1,b)\frac{|b|^{p-k}+b^{p-k}}{1+\abs{b}^{p-1}}  \diff b
	\\
	& =  \frac{p!}{(p-k)!}\int_{\R} \mu_+(b) \frac{|b|^{p-k}}{1+\abs{b}^{p-1}}
	- \mu_-(b)\frac{b^{p-k}}{1+\abs{b}^{p-1}}  \diff b.
	\end{align*}
	Therefore, 
	\begin{equation}\label{eq:Cond_K_even}
	\int_{\R} \frac{\mu_-(b)}{\psi(b)}b^{p-k}  \diff b
	=\frac{1}{p!}\int_{\R} f^{(p+1)}(b)|b|^{p-k}  \diff b - \frac{(p-k)!}{p!}f^{(k)}(0).
	\end{equation}
    Instead if $k$ is an odd number, then using the fact that $[x]_+^{2n} = x^{2n} - [-x]_+^{2n}$
		for any $n\in \mathbb{N}$ and any $x\in \mathbb{R}$,
    we get
    \begin{align*}
    f^{(k)}(0) &= \frac{p!}{(p-k)!}\int_{\R} \mu(1,b) \frac{[-b]_+^{p-k}}{1+\abs{b}^{p-1}}
	- \mu(-1,b)\frac{[b]_+^{p-k}}{1+\abs{b}^{p-1}}  \diff b
	\\
    & =  \frac{p!}{(p-k)!}\int_{\R} \mu(1,b) \frac{b^{p-k} - [b]_+^{p-k}}{1+\abs{b}^{p-1}}
	- \mu(-1,b)\frac{b^{p-k}-[-b]_+^{p-k}}{1+\abs{b}^{p-1}}  \diff b
	\\
	& =  \frac{p!}{(p-k)!}\int_{\R} \frac{\mu_-(b)}{\psi(b)}
	b^{p-k} \diff b
	- \frac{p!}{(p-k)!}\int_{\R} \frac{1}{\psi(b)}(\mu(1, b)[b]_+^{p-k} - \mu(-1, b)[-b]_+^{p-k} )\diff b.
	\end{align*}
	Moreover, using the fact that $\mu_+(b) - \mu_-(b) = 2\mu(-1, b)$, we have
    \begin{align*}
        \int_{\R} &\frac{1}{\psi(b)}\left(\mu(1, b)[b]_+^{p-k} - \mu(-1, b)[-b]_+^{p-k} \right)\diff b
        = \int_{\R} \frac{1}{\psi(b)}\left(\mu(1, b)[b]_+^{p-k} - \mu(-1, b)[-b]_+^{p-k} \right)\diff b
        \\
        & =\int_{\R} \frac{1}{\psi(b)}\left(\mu(1, b)[b]_+^{p-k} 
        +\mu(-1, b)[b]_+^{p-k} 
        -\mu(-1, b)[b]_+^{p-k} 
        - \mu(-1, b)[-b]_+^{p-k} \right)\diff b
        \\
        & = \int_{\R} \frac{1}{\psi(b)}\left(\mu_+(b)[b]_+^{p-k} 
        -\mu(-1, b)\underbrace{([b]_+^{p-k} + [-b]_+^{p-k})}_{b^{p-k}} \right)\diff b
        \\
        & = \int_{\R} \frac{1}{\psi(b)}\left(\mu_+(b)[b]_+^{p-k} 
        -\half \left(\mu_+(b) - \mu_-(b)\right)b^{p-k} \right)\diff b.
	\end{align*}
	Then for odd $k$ we conclude that 
	\begin{align*}
    f^{(k)}(0) &=
     \frac{p!}{(p-k)!}\left(\int_{\R} \frac{\mu_-(b)}{\psi(b)}
	b^{p-k} \diff b
	- \int_{\R} \frac{1}{\psi(b)}\left(\mu_+(b)[b]_+^{p-k} 
        -\half \left(\mu_+(b) - \mu_-(b)\right)b^{p-k} \right)\diff b
        \right)
        \\
        &=\frac{p!}{(p-k)!}\left(\half\int_{\R} \frac{\mu_-(b)}{\psi(b)}
	b^{p-k} \diff b
	- \int_{\R} \frac{1}{\psi(b)}\left(\mu_+(b)[b]_+^{p-k} 
        -\half \mu_+(b)b^{p-k} \right)\diff b
        \right)
        \\
        &= \frac{p!}{(p-k)!}\left(\half\int_{\R} \frac{\mu_-(b)}{\psi(b)}
	b^{p-k} \diff b
	+ \int_{\R} \frac{\mu_+(b)}{\psi(b)}[-b]_+^{p-k} 
         \diff b \right).
	\end{align*}
	Finally, we get
	\begin{equation}\label{eq:Cond_K_odd}
	\int_{\R} \frac{\mu_-(b)}{\psi(b)}
	b^{p-k} \diff b  =  2 \frac{(p-k)!}{p!}f^{(k)}(0) - \frac 2{p!} \int_{\R} f^{(p+1)}[-b]_+^{p-k} 
         \diff b.
	\end{equation}
 
    In summary, any measure $\mu_-$ that fulfills  \eqref{eq-const_alpham},
		\eqref{eq:Cond_K_even}, \eqref{eq:Cond_K_odd},
    in conjunction with $p!\frac{\mu_+}{\psi}=f^{(p+1)}$ and a suitable
	$\intbias$, gives $f=\mathtt{H}_{\mu,\intbias}^p$.
    Using \eqref{eq:mu}, we have
    \begin{equation}\label{eq:objalphaminus}
        \norm{\mu}_{\mathcal{M}^1(1/\psi)} =       
        \frac12 \int_\R \left(\abs{ \frac 1 {p!} f^{(p+1)}(b) + \frac {\mu_-(b)}{\psi(b)}}
        +  \abs{ \frac 1 {p!} f^{(p+1)}(b) -  \frac {\mu_-(b)}{\psi(b)}}\right) \diff b.
    \end{equation}
    In order to minimize $\norm{ \mu}_{\mathcal{M}^1(1/\psi)}$ we have to solve the following 
    \begin{align}\label{eq:alphaminusproblem}
    \begin{split}
        \min_{\mu_-} & \frac12 \int_\R \left(\abs{ \frac 1 {p!} f^{(p+1)}(b) + \frac {\mu_-(b)}{\psi(b)}}
        +  \abs{ \frac 1 {p!} f^{(p+1)}(b) -  \frac {\mu_-(b)}{\psi(b)}}\right) \diff b
        \\
        &\textrm{such that } \left\{\begin{array}{ll}
        \quad \frac 1{p!}\left(f^{(p)}(-\infty) + f^{(p)}(\infty)\right) &=  \displaystyle\int_{\R} \frac{\mu_-(b)}{\psi(b)} \diff b.
        \\[1ex]
        \frac{1}{p!}\sum_{k=1}^{\frac{p-1}{2}}\displaystyle\int_{\R} f^{(p+1)}(b)|b|^{p-2k}  \diff b
		- \frac{(p-2k)!}{p!}f^{(2k)}(0) &= \sum_{k=1}^{\frac{p-1}{2}}
        \displaystyle\int_{\R} \frac{\mu_-(b)}{\psi(b)}b^{p-2k}  \diff b
        \\[1ex]
        \sum_{k=0}^{\frac{p-3}{2}}2 \frac{(p-(2k+1))!}{p!}f^{(2k+1)}(0) - \frac 2{p!} \displaystyle\int_{\R} f^{(p+1)}[-b]_+^{p-(2k+1)} 
         \diff b
         &= \sum_{k=0}^{\frac{p-3}{2}}\displaystyle\int_{\R} \frac{\mu_-(b)}{\psi(b)} b^{p-(2k+1)} \diff b.
        \end{array}
        \right.
    \end{split}
    \end{align}
    %
    %
    
    We let $\lambda_1, \lambda_2, \lambda_3 \in\R$
	be the Lagrange multipliers and $\mathcal{L}$ is the Lagrangian, given by
    \begin{multline*}
    \mathcal L (\mu_-, f, f', \dots, f^{(p+1)},\lambda_1, \lambda_2, \lambda_3) =
    \frac12 \int_\R \left(\abs{ \frac 1 {p!} f^{(p+1)}(b) + \frac {\mu_-(b)}{\psi(b)}}
        +  \abs{ \frac 1 {p!} f^{(p+1)}(b) -  \frac {\mu_-(b)}{\psi(b)}}\right) \diff b
        \\
    	+ \lambda_1\left(\int_{\R} \frac{\mu_-(b)}{\psi(b)} \diff b -  \frac{f^{(p)}(-\infty) - f^{(p)}(\infty)}{p!} \right)
    \\
    +\lambda_2\left(\sum_{k=1}^{\frac{p-1}{2}}\int_{\R} \frac{\mu_-(b)}{\psi(b)}b^{p-2k}  \diff b
    -\frac{1}{p!}\int_{\R} f^{(p+1)}(b)|b|^{p-2k}  \diff b + \frac{(p-2k)!}{p!}f^{(2k)}(0)\right)
    \\[1ex]
    +\lambda_3\left(\sum_{k=0}^{\frac{p-3}{2}}\int_{\R} \frac{\mu_-(b)}{\psi(b)} b^{p-(2k+1)} \diff b
    -2 \frac{(p-(2k+1))!}{p!}f^{(2k+1)}(0) + \frac 2{p!} \int_{\R} f^{(p+1)}[-b]_+^{p-(2k+1)} 
         \diff b \right).
    \end{multline*}
    In order to determine the functional derivative of $\mathcal L$ 
    with respect to $\mu_-$, we  recall that 
    a functional derivative of $\mathcal L$  is a functional of functional that we denote by
    $$
    \partial \mathcal L(\mu_-, f, f', \dots, f^{(p+1)},\lambda_1, \lambda_2, \lambda_3)[h]
    = \frac{d}{d\gamma}
        L(\mu_- + \gamma h, f, f', \dots, f^{(p+1)},\lambda_1, \lambda_2, 
        \lambda_3)\big|_{\gamma=0}.
    $$
    Hence in our case, the functional derivative of $\mathcal L$
    with respect to $\mu_-$ is the functional 
    that maps  the functions $\mu_-, f, f', \dots, f^{(p+1)},\lambda_1, \lambda_2, \lambda_3$
    and $h$ to the number
    $\frac{\partial
    \mathcal L}{\partial \mu_-(b)}(\mu_-, f, f', \dots, f^{(p+1)},\lambda_1, \lambda_2, \lambda_3)[h].$
    Moreover, we choose the preceding function $h$ to be the Dirac delta $\delta_b(t)$, hence
    we get 
    \begin{multline*}
    \frac{\partial \mathcal L}{\partial \mu_-(b)} (\mu_-, f, f', \dots, f^{(p+1)},\lambda_1, \lambda_2, \lambda_3) 
    =
    \\
    \frac12 \int_\R \left(\sign\prn{ \frac 1 {p!} f^{(p+1)}(t) + \frac {\mu_-(t) }{\psi(t)}}
        +  \sign\prn{ \frac 1 {p!} f^{(p+1)}(t) -  \frac {\mu_-(t)}{\psi(t)}}\right) \delta_b(t)\diff t
    	\\+ \lambda_1\left(\int_{\R} \frac{\delta_b(t)}{\psi(t)} \diff t \right)
    +\lambda_2\left(\sum_{k=1}^{\frac{p-1}{2}}\int_{\R} \frac{\delta_b(t)}{\psi(t)}t^{p-2k}  \diff t\right)
    +\lambda_3\left(\sum_{k=0}^{\frac{p-3}{2}}\int_{\R} \frac{\delta_b(t)}{\psi(t)} t^{p-(2k+1)} \diff t \right).
    \end{multline*}
   Furthermore, we set the functional derivative of the Lagrangian
	$\mathcal L$ with respect to $\mu_-$ to zero, that is 
    \begin{multline}\label{eq:L0}
        0 \in \frac{\partial \mathcal L}{\partial \mu_-(b)}(\mu_-, f, f', \dots, f^{(p+1)},
	\lambda_1, \lambda_2, \lambda_3)
        =
        \\
        \half\left( \sign\prn{\frac 1 {p!} f^{(p+1)} + \frac {\mu_-}{\psi}}
        - \sign\prn{ \frac 1 {p!} f^{(p+1)} -  \frac {\mu_-}{\psi}}
        \right)
        +\frac{\lambda_1}{\psi}
        + \frac{\lambda_2}{\psi}\sum_{k=1}^{\frac{p-1}{2}}b^{p-2k}
        + \frac{\lambda_3}{\psi}\sum_{k=0}^{\frac{p-3}{2}}b^{p-(2k+1)}.
    \end{multline}
    Let $\Lambda = \frac{\lambda_1}{\psi}
        + \frac{\lambda_2}{\psi}\sum_{k=1}^{\frac{p-1}{2}}b^{p-2k}
        + \frac{\lambda_3}{\psi}\sum_{k=0}^{\frac{p-3}{2}}b^{p-(2k+1)}$.
    Therefore the possible values of
    $\Lambda$ depend on $\lambda_1, \lambda_2,\lambda_3$ and $b$,
    which can be treated as follows:
    \begin{description}
        \item[$\Lambda = 0$]: 
            Then $\sign\prn{\frac 1 {p!} f^{(p+1)} + \frac {\mu_-}{\psi}} 
            = 
            \sign\prn{\frac 1 {p!} f^{(p+1)} - \frac {\mu_-}{\psi}}$,
	which implies that  $\frac{\abs{\mu_-}}{\psi} \leq \frac 1{p!}\abs{f^{(p+1)}}$ pointwise
            and from \eqref{eq:objalphaminus} we have
            $$
            \norm{{\mu}}_{\mathcal{M}^1(1/\psi)} = \frac 1{p!}     
        \int_\R \abs{f^{(p+1)} (b)} \diff b.
        $$
 
            Furthermore,  in order to satisfy \eqref{eq:alphaminusproblem}, we have
            \begin{align*}
        \frac 1{p!} \abs{f^{(p)}(-\infty) + f^{(p)}(\infty)} 
        =
        \abs{\int_\R \frac{{\mu_-(b)}}{\psi(b)}\diff b}
        \leq  \int_\R \frac{\abs{\mu_-(b)}}{\psi(b)}\diff b \leq  \frac 1{p!} \int_{\R} \abs{f^{(p+1)}(b)} \diff b. 
    \end{align*}
    and
    \begin{align*}
     \frac{1}{p!}\sum_{k=1}^{\frac{p-1}{2}}\abs{\int_{\R} f^{(p+1)}(b)|b|^{p-2k}  \diff b -
		\frac{(p-2k)!}{p!}f^{(2k)}(0) }
     &\leq \sum_{k=1}^{\frac{p-1}{2}}\int_{\R} \frac{\abs{\mu_-(b)}}{\psi(b)}\abs{b}^{p-2k}  \diff b
     \\
     &\leq \frac{1}{p!}\sum_{k=1}^{\frac{p-1}{2}}
     \int_{\R} \abs{f^{(p+1)}(b)b^{p-2k}} \diff b
     \end{align*}
     moreover, 
    \begin{align*}
     \sum_{k=0}^{\frac{p-3}{2}}2 \frac{(p-(2k+1))!}{p!}\abs{f^{(2k+1)}(0) - \frac 2{p!}
	\displaystyle\int_{\R} f^{(p+1)}[-b]_+^{p-(2k+1)} 
         \diff b}
         &\leq  \sum_{k=0}^{\frac{p-3}{2}}\displaystyle\int_{\R} \frac{\abs{\mu_-(b)}}{\psi(b)} b^{p-(2k+1)} \diff b
         \\
     &\leq \frac{1}{p!}\sum_{k=0}^{\frac{p-3}{2}}
     \int_{\R} \abs{f^{(p+1)}(b)}b^{p-(2k+1)} \diff b.
     \end{align*}
    
      \item[$\Lambda<0$]:
        In this case $\frac 1 {p!} f^{(p+1)} + \frac {\mu_-}{\psi} \geq 0$
        and $\frac 1 {p!} f^{(p+1)} - \frac {\mu_-}{\psi} \leq 0$ pointwise,        
        then $\frac{\mu_-}{\psi} \geq \frac 1 {p!}\abs{f^{(p+1)}}$
        and from \eqref{eq:objalphaminus} and the constraint in \eqref{eq:alphaminusproblem}:
        we get 
        $\norm{{\mu}}_{\mathcal{M}^1(1/\psi)} = \int_\R\frac{ \mu_-(b)}{\psi(b)} \diff b
        =\frac{1}{p!}\left( f^{(p)}(-\infty) + f^{(p)}(\infty)\right)$, this happens if
        $f^{(p)}(-\infty) + f^{(p)}(\infty) \geq  \int_{\R} \abs{f^{(p+1)}(b)} \diff b$
        and
      \begin{align*}
     \sum_{k=0}^{\frac{p-3}{2}}2 \frac{(p-(2k+1))!}{p!}{f^{(2k+1)}(0) - \frac 2{p!} 
	\displaystyle\int_{\R} f^{(p+1)}[-b]_+^{p-(2k+1)} 
         \diff b}
         &=  \sum_{k=0}^{\frac{p-3}{2}}\displaystyle\int_{\R} \frac{{\mu_-(b)}}{\psi(b)} b^{p-(2k+1)} \diff b
         \\
     &\geq  \frac{1}{p!}\sum_{k=0}^{\frac{p-3}{2}}
     \int_{\R} \abs{f^{(p+1)}(b)}b^{p-(2k+1)} \diff b,
     \end{align*}
     where we used the fact that $p-(2k+1)$ is an even number.
     
     \item [$\Lambda >0$]:
          In this case $\frac 1 {p!} f^{(p+1)} + \frac {\mu_-}{\psi} \leq 0$
        and $\frac 1 {p!} f^{(p+1)} - \frac {\mu_-}{\psi} \geq 0$ pointwise,
	then $\frac{\mu_-}{\psi} \leq -\frac 1 {p!}\abs{f^{(p+1)}}$
        and from \eqref{eq:objalphaminus} and the constraint in \eqref{eq:alphaminusproblem}:
        we get 
        $$
	\norm{{\mu}}_{\mathcal{M}^1(1/\psi)} = \int_{\R}(-\frac{ \mu_-}{\psi})\diff b =
         -\frac{1}{p!}(f^{(p)}(-\infty) + f^{(p)}(\infty))
	$$
	 which is true when $f^{(p)}(-\infty) + f^{(p)}(\infty) \leq -\int_{\R} \abs{f^{(p+1)}(b)} \diff b$.
    \end{description}
   In view of the previous cases, we conclude that 
   $$
   \norm{{\mu}}_{\mathcal{M}^1(1/\psi)}
   =
   \frac1{p!}\max( \int_{\R} \abs{f^{(p+1)}}\diff b, \abs{f^{(p)}(-\infty)+f^{(p)}(+\infty)} ).
   $$
%
\end{proof}
%
%
%
\section{Infinite width bounded norm networks for functions
in higher dimension}\label{sec:rnorm}

The main aim of our paper is to determine  $\Rbar{f}$ for any function
$f : \mathbb{R}^d \rightarrow \R$,
and characterize when it is finite in order to know what are the  functions that
can be approximated arbitrarily well with bounded norm
but unbounded width RePU neural networks.

\par

For odd $p$ every two layers RePU networks decompose into the sum
of a network with the $p$ power of  absolute units plus a polynomial part of order $p$.
Indeed, if $p\in \mathbb{N}$ is odd, then $\relu{x}^p=\half\left(|x|^p + x^p\right)$.
Given the proof of Theorem \ref{thm:main_1d}, \ie in the univariate case,
the weights on the $p$ power of the absolute value units determine the representational cost,
with a correction term needed if the  polynomial part has a large weight.

In our approach we propose to consider adding \emph{unregularized monomial units}
$\sum_{k=1}^{p} \scal{ v_k}{ x^k}$
to ``absorb'' any representational cost due to the polynomial part. 
In other words, for any $\theta\in\Theta$ and $\vv\in\mathbb{R}^{d\times p}$
we define the class of unbounded width two layers RePU neural networks
$g_{\theta, \vv, p}$ with monomial units
by
$$g_{\theta,\vv, p}(x) = g_{\theta, p}(x) + \sum _{k=1}^p\scal{ v_k}{ x^k},
$$
where $g_{\theta, p}$ is as defined in \eqref{eq:finitenet}.
Moreover we associate $g_{\theta,\vv, p}$ with
the same weight norm $C(\theta)$ as defined in \eqref{eq:c}.
Namely, we exclude the norm of the weight $\vv = (v_1, \dots,, v_p)$
on the additional units from the cost.
Then  we define the representational cost $\Rbarone{f}$ for this class
of neural networks in the  following way
\begin{equation}\label{eq:rbar}
  \Rbarone{f} :=  \lim_{\varepsilon \rightarrow 0} \left(\inf_
  {\substack{\theta \in \Theta\\ \vv \in \mathbb{R}^{d\times p}}}
   C(\theta):| g_{\theta,\vv, p}(x) - f(x)|
   \leq \varepsilon,~\forall~\|x\| \leq 1/\varepsilon,
   ~\text{ and }~g_{\theta, \vv, p}( 0) = f( 0)\right).
\end{equation}
Likewise, for all $\mu\in \cM^1({\mathbb{S}^{d-1}\times \R})$, $\vv\in\mathbb{R}^{d\times p}$, $c\in\R$,
we define an infinite width network with monomial units by
$$
\mathtt{H}_{\mu,\vv,c}^p(x) := \mathtt{H}_{\mu,c}^p(x) + \sum_{k=1}^p \scal{ v_k}{ x^k}.
$$
Let $\psi(x)= 1+|x|^{p-1}$ where $p$ is an odd number,
we prove in Appendix \ref{appsec:opt} that $\Rbarone{f}$
can be expressed as follows:
\begin{align}\label{eq:opt1}
  \Rbarone{f} & =
  \min_{\mu \in \cM^1({\mathbb{S}^{d-1}\times \R}),\vv\in\mathbb{R}^{d\times p},c\in \R} 
  \|{\mu}\|_{\mathcal{M}^1(1/\psi)}^{\nicefrac 1p}
  \text{ such that }
  f  = \mathtt{H}_{\mu,\vv,c}^p.
\end{align}

\par

Consequently, we show that the minimizer of \eqref{eq:opt1} is unique
and can be determined in the next lemma, where its  proof is given in Appendix \ref{appsec:thm1}.

\begin{lemma}\label{lem:uniqueness}
Let $p\in \mathbb{N}$ be an odd number, then
 $\Rbarone{f} = \|{\mu^+}\|_{\mathcal{M}^1(1/\psi)}^{\nicefrac 1p}$
 where $\mu^+\in \mathcal{M}_e^1({\mathbb{S}^{d-1}\times \R})$
 is the unique \emph{even measure}\footnote{Roughly speaking, a measure
 $\mu$ is even if
 $\mu(w,b) = \mu(-w,-b)$ for all $(w,b)\in{\mathbb{S}^{d-1}\times \R}$.}
such that $f = \mathtt{H}_{\mu^+,\vv,c}^p$ for some
$\vv\in\mathbb{R}^{d\times p}$, $c\in\R$.
\end{lemma}

\par

\subsection{Representational cost for infinite width RePU neural networks}

We now observe that the Laplacian of an infinite width network can be related to
the dual Radon transform. First, we consider a measure that has a density $ \mu$ belongs
to Schwartz class $\mathscr{S}(\SR)$.
Note that, compared to \cite{Ongie2019}, we can relax the  assumptions on the density,
since the weight $\frac1{1+ |b|^{p-1}}$ plays an important role in our construction.
Then, for an odd number $p$, we consider the following
\begin{equation}\label{eq:f_density}
	f(x) = \int_{\S^{d-1}\times \R}\frac{[\scal wx-b]_+^p -[-b]_+^p}{1+|b|^{p-1}}
		\mu(w,b)\,dw\,db + \sum_{k=1}^p \scal {v_k}{x^k} + c.
\end{equation}
Taking a differentiation of  \eqref{eq:f_density}, with respect to $x$,
$(p+1)$ times inside the integral, by applying the Laplacian $(p+1)/2$ times, we get
\begin{equation}\label{eq:netlaplacian}
    \Delta^{\frac{p+1}{2}} f(x) =\int_{\SR} p! \frac{\delta(\scal wx -b)}{1+|b|^{p-1}}\mu(w,b)\,dw\,db 
                =p!  \int_{\S^{d-1}} \frac{\mu(w,\scal wx)}{1+|\scal wx|^{p-1}}\,dw
\end{equation}
where $\delta(\cdot)$ denotes a Dirac delta.
In view of the dual Radon transform \eqref{eq:adjoint} and the previous
identity \eqref{eq:netlaplacian}, 
$ \Delta^{\frac{p+1}{2}} f$ can be considered as a dual Radon transform
of the weighted density $p! \frac{\mu(w,\scal wx)}{1+|\scal wx|^{p-1}}$.
That is, we have the following identity
\begin{equation}\label{eq:identity_laplacian_radon}
 \Delta^{\frac{p+1}{2}} f = p!\mathscr{R}^*\{ \frac\mu\psi\},
\text{ where } \psi(b) = 1+|b|^{p-1} \text{ and } b \in \R .
\end{equation}

Using the characterization of $\Rbarone{f}$ given in Lemma \ref{lem:uniqueness},
 and the inversion dual Radon transform given
in \eqref{eq:invform} to \eqref{eq:identity_laplacian_radon},
we can determine  $\Rbarone{f}$ for a given function $f$ as next result shows.

\begin{lemma}\label{lem:smooth}
	Let $p\in \mathbb{N}$ be an odd number,
	$f= \mathtt{H}_{\mu,\vv,c}^p$
	where $\mu$ is an even measure given by a density in $\mathscr{ S } \left( \S^{d-1}\times\R \right)$
	and $\vv=(v_1, \dots,, v_p)\in\R^{d\times p}$, $c\in\R$
	and let $\psi(x)= 1+|x|^{p-1}$ .
	Then 
	$$
	\frac\mu\psi = \frac{\gamma_d}{p!} \Rdn{(-\Delta)^{(d+p)/2}f},
	\text{ and }\quad
	\Rbarone{f} =\|\frac{ \gamma_d}{p!}\Rdn{(-\Delta)^{(d+p)/2}f}\|_{\mathcal{M}^1}^{\nicefrac 1p}
	$$
	where $\gamma_d = \frac{1}{2(2\pi)^{d-1}}$.
\end{lemma}

Using the fact that $\Rdn{\Delta^{(d+1)/2}f} = \partial_b^{d+1}\Rdn{f}$, and that 
$\Rbarone{f}$ can be computed in terms of the ${\mathcal{M}^1}$ norm
of $ \partial_b^{d+1}\Rdn{f}$ we discuss the question of its the convergence in the sequel.
\begin{proposition}\label{prop:ell1}
	 Let $d, p\in \mathbb{N}$ be odd numbers,
	 if  $f \in L^1(\Rd)$ and $\Delta^{(d+p)/2}f \in L^1(\Rd)$, then
	\begin{equation}\label{eq:rnormclassical}
	 	 \Rbarone{f} =
	 	 \norm{\frac{\gamma_d}{p!}\Rdn{\Delta^{(d+p)/2}f}}_{\mathcal{M}^1}^{\nicefrac 1p}
	 	 =\norm{ \frac{\gamma_d}{p!}\partial_b^{d+p}\Rdn{f}}_{\mathcal{M}^1} ^{\nicefrac 1p}< \infty.
	\end{equation}
\end{proposition}
\begin{proof}
	Recalling that the  Radon transform is a bounded linear operator from $L^1(\Rd)$ to $L^1(\SR)$
	\cf \cite{boman2009support}.
	Then, the fact that  $\Delta^{(d+p)/2} f \in L^1(\Rd)$
	ensures that $\Rdn{\Delta^{(d+p)/2} f} \in L^1(\SR)$.
	
	Therefore, by Definition \ref{def:Rnorm} of $\Rnorm{f}$ we have
	\begin{align*}
		  \Rnorm{f}&
		  	= \sup\{ \prn{-\frac{\gamma_d}{p!}\langle f, (-\Delta)^{(d+p)/2}
		        \Rdns{\phi}\rangle}^{1/p} 
		  		: \phi\in \Sc_e(\SR), \|\phi\|_\infty \leq 1 \}\\
		            	&
			= \sup\{\prn{ \langle -\frac{\gamma_d}{p!} \Rdn{(-\Delta)^{(d+p)/2} f},\phi \rangle}^{1/p} 
		            		: \phi\in C_{0,e}(\SR), \|\phi\|_\infty \leq 1 \}
		\end{align*}
	where we used the fact that the Schwartz class $\Sc_e(\SR)$ is dense in $C_{0,e}(\SR)$.
	If we assume that $\mu \in \cM_e(\SR)$
	given by a  density then in view of Lemma \ref{lem:smooth} we have
	$\| \mu \|_{\cM^1(1/\psi)} = \frac{\gamma_d}{p!}\|\Rdn{(-\Delta)^{(d+p)/2} f}\|_{\cM^1}$.
	Considering the dual definition of the total variation norm \eqref{eq:tvnorm}
	it follows that
	\begin{align*}
	    \Rnorm{f}
			&= \sup\{\prn{ \langle \frac \mu\psi, \phi \rangle}^{1/p} :
			\phi\in C_{0, e}(\SR), \|\phi\|_\infty \leq 1 \}\\
		           	 &= \| \mu\|_{\cM^1(1/\psi)}
		  = \prn{\|\frac{\gamma_d}{p!}\Rdn{(-\Delta)^{(d+p)/2}f}\|_{\cM^1}}^{1/p}.
	\end{align*}
	 In view of \eqref{eq:interwining_ppt}, if $f \in L^1(\Rd)$, then, for $s=d+p$,
	$\Rdn{(-\Delta)^{(d+p)/2}f} = (-\partial_b)^{d+p}\Rdn{f}$
	which gives $\Rnorm{f} =  \prn{\|\frac{\gamma_d}{p!}(-\partial_b)^{d+p}\Rdn{f}\|_{\cM^1}}^{1/p}$.
\end{proof}
The  definition of $\Rbarone f$ suffers from functions $f$ that are non-integrable along hyperplanes
or non-smooth,
therefore we need to extend the equalities in \eqref{eq:rnormclassical} to a more general case.
Mainly, thanks to a duality argument we define a functional ``$\cR$-norm''
that extends \eqref{eq:rnormclassical}
to the case where $f$ is possibly non-smooth or not integrable along hyperplanes,
which is the case if
$f$ is a finite width ReLU neural network, that is $p=1$.
\begin{definition}
	Let  $Lip^\kappa(\Rd)$ denotes the space of $\kappa$-order Lipschitz functions on $\Rd$
	for $\kappa\in \mathbb{N}_0$ in the sense that  there exists $\alpha \in \mathbb{N}_0^d$
	such that $f\in C^\kappa(\Rd)$ and $\partial^\alpha f$ is Lipschitz function
	where $\kappa$ is the smallest integer
	satisfies $|\alpha|= \kappa$.
\end{definition}
Note that, in view of the previous definition, the ReLU is a $0$-order Lipschitz function.
For any $f\in \Lip^\kappa(\Rd)$, we let
$\|f\|_L :=\sup_{|\alpha |= \kappa} \sup_{x\neq y}
	\frac{|\partial^\alpha f(x)- \partial^\alpha f(y)|}{\|x-y\|}$,
be the smallest possible Lipschitz constant for $\partial ^\alpha f$ such that $|\alpha| = \kappa$.
\begin{remark}\label{prop:lip}
Infinite width networks considered in our paper are $\kappa$-order Lipschitz functions.
\end{remark}

Next we define a functional on the space of all $\kappa$-order Lipschitz continuous functions,
where $\kappa<p$. Mainly the aim of the following definition is to 
re-express the $L^1$-norm in \eqref{eq:rnormclassical} as a supremum
of the inner product over a space of dual functions $\phi:\SR\rightarrow \R$,
then restrict $\phi$ to a space where $\Delta^{(d+p)/2}\Rdns{\phi}$ is always well-defined.
Where the restriction comes from the fact that 
$$
  \|\Rdn{\Delta^{(d+p)/2}f}\|_{\mathcal{M}^1} = 
  \sup_{\|\phi\|_\infty \leq 1} \langle \Rdn{\Delta^{(d+p)/2}f}, \phi \rangle 
= \sup_{\|\phi\|_\infty \leq 1} \langle f , \Delta^{(d+p)/2}\Rdns{\phi} \rangle,
$$
where we use the fact that $\cR^*$ is the adjoint of $\cR$ and the Laplacian $\Delta$ is self-adjoint.
\begin{definition}\label{def:Rnorm}
	Let $p\in \mathbb{N}$ be an odd number, $\kappa\in \mathbb{N}_0$ such that $\kappa<p$,
	then for any $\kappa$-order Lipschitz function $f:\Rd\rightarrow \R$
	we define its \emph{$\cR$-norm}
	\footnote{Strictly speaking, the functional $\Rnorm{\cdot}$ is not a norm,
	but it is a semi-norm on the space of functions for which it is finite;
	more details can be found in the Appendix.}
	by
\begin{equation}\label{eq:dualdef}
	\Rnorm{f} := \sup\left\{\prn{-\frac{\gamma_d}{p!}\langle f, (-\Delta)^{(d+p)/2}
	\Rdns{\phi}\rangle}^{\nicefrac 1p} :
	\phi \in \Sc(\SR), \phi \text{~even~}, \|\phi\|_\infty \leq 1 \right\}.
\end{equation}
	where $\gamma_d = \frac{1}{2(2\pi)^{d-1}}$
	and $\langle f,g\rangle := \int_{\Rd} f(x)g(x) dx$.
	If $f$ is not a $\kappa$-order Lipschitz function we set $\|f\|_{\cR} = +\infty$.
\end{definition}
In Appendix \ref{appsec:thm1} we show that the $\cR$-norm is well-defined,
though not always finite, for any $\kappa$-order Lipschitz functions
where $\kappa<p$.
The main result, in higher dimension, in this paper is the following theorem,
where its proof is in Appendix \ref{appsec:thm1}.

\begin{theorem}\label{thm:main}
	Let $d, p\in \mathbb{N}$ be an odd numbers, $\kappa\in \mathbb{N}_0$ such that $\kappa<p$,
	then
	$\prn{p!}^{\nicefrac 1p}\Rbarone{f} = \Rnorm{f}$ for any real valued function $f$ defined on $\R^d$.
	In particular, $\Rbarone{f}$ is finite if and only if $f$ is $\kappa$-order Lipschitz
	and $\Rnorm{f}$ is finite.
\end{theorem}

\bibliography{literature}
\bibliographystyle{abbrv}

\appendix
\part*{Appendix}
\section{Equivalence of overall control to control on the output layer}
\label{app:neyshaburreproof}
For the sake of completeness, we prove that regularizing the $\ell^{\nicefrac{2}{p}}$
quasinorm of the weights in the output layer and the $\ell_2$
norm of weights of the hidden layer, (\ie $C(\theta)$)
is equivalent to restricting the $\ell_2$ norm of incoming weights
for each unit in the hidden layer,
and regularizing the $\ell_{\nicefrac{1}{p}}$ quasinorm of weights in the output layer. 
The idea behind the proof was first used in \cite[Theorem 1]{neyshabur2014search}
for networks without an unregularized bias. Moreover, \cite[Theorem 10 ]{neyshabur2015norm}
 extended the equivalence for convex neural networks.
 In, \cite[Lemma  A.1]{savarese2019infinite}
the authors  showed that the equivalence is valid for their setting.
A similar argument works in our case with the RePU activation function as $p \in \mathbb{N}$,
it is clear that when $p=1$ then our result is the same as \cite[Lemma  A.1]{savarese2019infinite}.

Recall the definition of two layers RePU networks $g_{\theta, p}$ for $\theta \in\Gamma$:
\begin{equation*}
	g_{\theta, p}(x) = \sum_{i=1}^k a_i \relu{w_ix - b_i} ^p
	+ 
	c
\end{equation*}

\par

\begin{lemma}
Let
$
\Gamma=\left\{\theta=\left(k, W, {b}, a,c \right) \mid
k \in \mathbb{N}, W \in \mathbb{R}^{k \times d}, {b} \in \mathbb{R}^{k},
a \in \mathbb{R}^{k}, c \in \mathbb{R}\right\}
$, then we have
	\begin{equation*}
		\begin{aligned}
			& {\inf_{\theta \in \Gamma}}
			& & ~{\frac{1}{2}} \sum_{i=1}^k \left( (a_i)^{\nicefrac{2}{p}}
				+ \norm{w_i} _2^2 \right) \quad = 
			& & {\inf_{\theta \in\Gamma}} 
			& & ~\norm{a}_{\nicefrac{1}{p}}^{\nicefrac{1}{p}}
			\\
			& ~\text{such that } 
			& & ~g_{\theta,p} = f 
			& & ~\text{such that }
			& & ~g_{\theta, p} = f~,~ \text{ for all } i: \norm{w_i}_2 = 1
		\end{aligned}
	\end{equation*}
	\label{lemma-pal}
\end{lemma}

\begin{proof}
    Let $\theta \in\Gamma$, we consider a rescale by a factor $r$
    (which we precise later) of $\theta$,
    that is we set $\tilde \theta= (k, \tilde W, \tilde b, \tilde a, c)$
    such that $\tilde w_i = r_i w_i $, $\tilde a_i = \frac{a_i}{(r_i)^{p}}$, $\tilde b_i = r_i b_i$.
     Now, check that, for all $i$:
    \begin{equation*}
    \begin{split}
	    \tilde a_i \relu{ \tilde w_i x - \tilde b_i}^p
	    =
	   \frac{a_i}{(r_i)^p} \relu{r_i \left(w_ix  - b_i \right)}^p
	    =a_i \relu{ w_i x  - b_i}^p.
	\end{split}
    \end{equation*}
	The previous  equality clearly shows that $g_{\theta, p}= g_{\tilde \theta, p}$. Moreover, we have that,
	from the inequality between arithmetic and geometric means:
    \begin{equation*}
    \begin{split}
	{\frac{1}{2}} \sum_{i=1}^k
    \left( (a_i)^{\nicefrac{2}{p}} + \norm{w _i} _2^2\right) 
	\geq 
	\sum_{i=1}^k
    \abs{a_i}^{\nicefrac{1}{p}} \cdot \norm{w _i} _2,
	\end{split}
    \end{equation*}
	obviously  a rescaling given by $r_i = \sqrt{\abs{a_i}^{\nicefrac{1}{p}} / \norm{w _i}_2 }$
	minimizes the left-hand side and achieves equality.
	Since the right-hand side is invariant to rescaling,
	we can arbitrarily set $\norm{w_i} _2 = 1$ for all $i$,
	then $\sum_{i=1}^k \abs{a_i}^{\nicefrac{1}{p}}
	            = \norm{a}_{\nicefrac{1}{p}}^{\nicefrac{1}{p}} $.
\end{proof}

\section{Characterization of representational cost}\label{appsec:opt}
Here we establish the optimization equivalents of the representational costs
$\Rbar{f}$ and $\Rbarone{f}$ given in \eqref{eq:opt0} and $\eqref{eq:opt1}$. 

As an intermediate step, we first give equivalent expressions for $\Rbar{f}$
and $\Rbarone{f}$ in terms of sequences finite width two layers ReLU networks
converging pointwise to $f$.

Next, we state our definition for a \emph{discrete} measure,
our definition is slightly different
than the one used in \cite{Ongie2019}.
Indeed we use an additional outer-weight
generated by a continuous function on $\R$.

\begin{definition}
	Let $\mathcal{D}_{\psi}(\SR)$ denotes the space of all measures given
	by a finite linear combination of Diracs, \ie all $\mu \in  \mathcal{M}^1(\SR)$ of
	the form $\mu = \sum_{i=1}^k a_i\psi(b_i) \delta_{(w_i,b_i)}$ for some
	$a_i \in \R$, $(w_i,b_i)\in \SR$, $\psi\in C(\R)$, $i=1,...,k$, where $\delta_{(w,b)}$
	denotes a Dirac delta at location $(w,b)\in\SR$.
	We call any  $\mu \in \mathcal{D}_\psi (\SR)$ a discrete measure. 
\end{definition}

One important property of discrete measures is the 
one-to-one correspondence between them
and finite width two layers RePU networks.
Namely, for any $\theta \in \Theta$ a finite width RePU network
$g_{\theta, p}(x) = \sum_{i=1}^k a_i [w_i x - b_i]_+^p + c$,
setting $\mu = \sum_{i=1}^k a_i (1+\abs{b_i}^{p-1})\delta_{(w_i,b_i)}$
we have $f = \mathtt{H}_{\mu,c'}^p=g_{\theta, p}$ with $c' = g_{\theta, p}( 0)$.
We denote the correspondence by
$\theta \in \Theta\leftrightarrow \mu \in \mathcal{D}_{1+\abs{b}^{p-1}}(\SR)$.
Consequently, in this situation
$$
\| {\mu}\|_{{\mathcal{M}^1(1/\psi)}} ^{\nicefrac{1}{p}}\leq
C(\theta) = \sum_{i=1}^k |a_i|^{\nicefrac 1p} \leq 
\| {\mu}\|_{{\mathcal{M}^1(1/\psi)}},
$$
where $\psi(b) = 1-|b|^{p-1}$ for any $b\in \mathbb{R}$.

In the following, we recall one of the notions of convergence provided with
$\mathcal{M}^1(\SR)$ and the definition of tight measure. 
\begin{definition}[Narrowly convergence in $\mathcal{M}^1(\SR)$]
A sequence of measures $\{\mu_n\}$, with $\mu_n \in  \mathcal{M}^1(\SR)$
is said to converge \emph{narrowly} to a measure $\mu \in  \mathcal{M}^1(\SR)$
if  $\int \varphi\, d\mu_n \rightarrow \int \varphi\, d\mu$ for all
$\varphi \in C_b(\SR)$, where $C_b(\SR)$ denote the set of all
continuous and bounded functions on $\SR$.
\end{definition}
\begin{definition}
A sequence of measures  $\{\mu_n\}$ is called \emph{tight} if for all $\varepsilon > 0$
there exists a compact set $K_\varepsilon \subset \SR$ such that
$|\mu_n|(K^c_\varepsilon) \leq \varepsilon$ for all $n$ sufficiently large.
\end{definition}
Next result can be found in \cite[Theorem 6.8]{malliavin2012} and
\cite[Theorem 8.6.2]{bogachev2007}.

\begin{lemma}
 Every narrowly convergent sequence  is tight.
Moreover, any sequence of measures $\{\mu_n\}$ that is tight and uniformly bounded
in total variation norm has a narrowly convergent subsequence.
\end{lemma}

Next statement contributes towards the main result of Appendix \ref{appsec:opt}
concerning equal expressions for the
representational costs $\Rbar{\cdot}$ and $\Rbarone{\cdot}$, \cf Lemma \ref{lem:opteq}.

\begin{lemma}\label{lem:rbareq}
Let $p\in\mathbb{N}$,  $\psi :\mathbb{R}\rightarrow \mathbb{R}$ such that $\psi(b) = 1-|b|^{p-1}$
and $f:\Rd\rightarrow \R$ let $ f_0$ denotes the function
$ f_0(x) =  f(x)-f(0)$.
For $\Rbar{f}$ as defined in \eqref{eq:rbar0}
and $\Rbarone{f}$ as defined in \eqref{eq:rbar}, we have
\begin{equation}\label{eq:rbarseq0}
    \Rbar{f} = \inf \left\{\limsup_{n\rightarrow\infty}
        \| {\mu_n}\|_{\mathcal{M}^1(1/\psi)}^{\nicefrac{1}{p}} :
    \mu_n \in \mathcal{D}_\psi (\SR), 
    ~~ \mathtt{H}_{\mu_n}^p\rightarrow  f_0
    ~\text{pointwise},~\{\mu_n\}~\text{tight}\right\}. 
\end{equation}
and
\begin{multline}\label{eq:rbarseq}
    \Rbarone{f} = \inf \Big\{\limsup_{n\rightarrow\infty}
            \| {\mu_n}\|_{\mathcal{M}^1(1/\psi)}^{\nicefrac{1}{p}}
    : \mu_n \in \mathcal{D}_\psi(\SR),\vv_n=(v_{n, 1}, \dots, v_{n,p})
            \text{ such that } v_{n, i}\in\Rd,
    \\[1ex] 
    ~~ \mathtt{H}_{\mu_n,\vv_n,0}^p\rightarrow  f_0
    ~\text{pointwise},~\{\mu_n\}~\text{tight}\Big\}.\qquad 
\end{multline}
\end{lemma}
\begin{proof}
First we prove \eqref{eq:rbarseq0} for $\Rbar{f}$. Similar arguments lead to
\eqref{eq:rbarseq} for $\Rbarone{f}$, therefore details are left for the reader.
We define  $ R_\varepsilon(f)$ such that
$\Rbar{f} = \lim_{\varepsilon \rightarrow 0}R_\varepsilon(f)$, thus
\begin{equation}
  R_\varepsilon(f) := \inf_{\theta \in \Theta} C(\theta\text{ such that }
\norm{g_{\theta, p} - f}_{L^{\infty}}
  \leq \varepsilon~\text{and} 
  ~~ g_{\theta, p}(0) = f(0).
\end{equation}

Also, let $L(f)$ denotes the right-hand side of \eqref{eq:rbarseq0}.

\par

Step 1:  We start by assuming that $\Rbar{f}$ is finite and we set $\varepsilon_n = 1/n$.
Then by the definition of $\Rbar{f}$  in \eqref{eq:rbar0},
for all $n$ there exists $\theta_n \in \Theta$
such that $C(\theta_n) \leq R_{\varepsilon_n}(f) + \varepsilon_n$, where 
$\norm{g_{\theta_n, p} - f}_{L^{ \infty}} \leq \varepsilon_n$
and $g_{\theta_n, p}( 0) = f( 0)$.
Moreover, in view of the correspondence between the parameters of a  network
and measures, we have $\theta_n \in \Theta$ $\leftrightarrow$ $\mu_n \in \mathcal{M}^1(\SR)$
with $g_{\theta_n, p} = \mathtt{H}_{\mu_n,c}^p$ where the outer-bias
$c = g_{\theta_n, p}( 0) = f( 0)$ and
$\| {\mu_n}\|_{\mathcal{M}^1(1/\psi)}^{\nicefrac{1}{p}}  \leq  C(\theta_n)$.

As a result of the previous correspondence, 
$\mathtt{H}_{\mu_n}^p(x) = g_{\theta_n,p}(x)-f( 0)$ and we have
$| \mathtt{H}_{\mu_n}^p(x) -  f_0(x)|
= | g_{\theta_n, p}(x) - f(x)| \leq \varepsilon_n$
for any $x\in \Rd$.
Therefore, $\mathtt{H}_{\mu_n}^p\rightarrow  f_0$
pointwise, and 
since $\mu_n\in \mathcal{D}_{\psi}({\SR})$ we have
\begin{equation}
  \limsup_{n\rightarrow\infty}
  \| {\mu_n}\|_{\mathcal{M}^1(1/\psi)}^{\nicefrac{1}{p}}
   \leq \limsup_{n\rightarrow\infty}
  (R_{\varepsilon_n}(f) + \varepsilon_n) = \Rbar{f}.
\end{equation}
In view of the previous arguments, we conclude that $L(f) \leq \Rbar{f}$.
Last step towards our result,
is to show that there exists a sequence of measures $\{\mu_n\}$
which is tight.
Therefore, let
$Q_{n, p}^{m}(x) =
    \partial^\gamma \int_{\SR} \frac{ (\scal{w}{x} - b )^p}{\psi(b)} d\mu_n(w, b)$
for any $m \in \mathbb{N}_0, p \in \mathbb{N}$ and $\gamma \in \mathbb{N}_0^d $
such that $|\gamma| =m$.
It is clear that $Q_{n, p}^{m}$  is well defined for all $p$ and $m$,
since $\mu_n \in \mathcal{D}_{\psi}(\SR)$.
In the case where $m =p-1$ then $Q_{n, p}^{p-1}$ is Lipschitz with
$\norm{Q_{n, p}^{p-1}}_L\leq p! \norm{{\mu_n}}_{\mathcal{M}^1(1/\psi)}
\leq B< \infty$, consequently
$\{Q_{n, p}^{p-1}\}$ is uniformly Lipschitz. Then by 
Arzela-Ascoli Theorem, $\{Q_{n, p}^{p-1}\}$ has a subsequence $Q_{n_k, p}^{p-1}$
that converges uniformly on compact subsets, hence $\{\mu_n\}$ is tight.

\par

Step 2: Let $L(f)$ finite and we fix any $\varepsilon > 0$.
Then by definition of $L(f)$ there exists a sequence 
$$
\mu_n\in \mathcal{D}_\psi(\SR)
\subset  \mathcal{M}^1(\SR)
\leftrightarrow \, \theta_n \in \Theta
$$
such that
$\lim_{n\rightarrow \infty} \| {\mu_n}\|_ {\mathcal{M}^1(1/\psi)}^{\nicefrac 1p}$ exists with
$\lim_{n\rightarrow \infty}\| {\mu_n}\|_{\mathcal{M}^1(1/\psi)}^{\nicefrac 1p} < L(f) + \varepsilon$,
then there exists an $N_1$ such that for all $n\geq N_1$ we have
$\| {\mu_n}\|_{\mathcal{M}^1(1/\psi)}^{\nicefrac 1p} \leq L(f) + \varepsilon$.
Moreover,
$f_n :=  \mathtt{H}_{\mu_n,c}^p =  g_{\theta_n, p}$
converges to $f$ pointwise where $c = f( 0)$
and satisfies $f_n( 0) = f( 0)$ for all $n$.

Since $\varepsilon$ is fixed then 
$\{x \in \mathbb{R}^d\;:\: \norm{x}\leq \nicefrac{1}{\varepsilon}\}$
is a compact set.
We can choose $\mu_n\in \mathcal{D}_\psi(\SR)$ such that $f_n$
is a monotonic sequence.
Hence Dini's theorem implies that $f_n$
converge to $f$ uniformly on compact subsets.
That is there exists an $N_2$ such that $\abs{f_n(x)- f(x)} \leq \varepsilon$
for all $\|x\|\leq 1/\varepsilon$ and $ f_n( 0) =  f( 0)$ for all $n\geq N_2$.
For all $n\geq N_2$, $f_n$ satisfies the constraints in the definition of $R_\varepsilon(\cdot)$.
Therefore, for all $n\geq \max\{N_1,N_2\}$ we have 
\begin{equation}
  0\leq R_\varepsilon(f)^p \leq C(\theta_n)^p \leq
    \|{\mu_n}\|_{\mathcal{M}^1(1/\psi)} \leq (L(f) + \varepsilon)^p.
\end{equation}
Hence
$R_\varepsilon(f) \leq L(f) + \varepsilon$,
if  $\varepsilon \rightarrow 0$, we get $\Rbar{f} \leq L(f)$.
Then, in view of Step 1 and Step 2,
$\Rbar{f}$ is finite if and only if $L(f)$ is finite
and $\Rbar{f} = L(f)$.
\end{proof}

Based on \cite[Theorem 6.9]{malliavin2012} and that
$(w, b)\rightarrow \frac{[\scal{w}{x}-b]_+^p-[-b]_+^p}{1+|b|^{p-1}}$
is continuous and bounded,
we can show the following lemma,
the proof is left for the reader.

\begin{lemma}\label{lem:narrow}
Let $p \in \mathbb{N}$, $f = \mathtt{H}_{\mu,\vv,c}^p$
for any $\mu\in \mathcal{M}^1(\SR)$, $\vv \in  \mathbb{R}^{d\times p}$
\ie $\vv = (v_1, \dots, v_p)$
such that $v_i\in\Rd,\, i=1,\dots,p$, $c\in\R$ and let $\psi(b)=1+|b|^{p-1}$ where $b\in \mathbb{R}$.
Then there exists a sequence of discrete measures
$\mu_n \in \mathcal{D}_\psi(\SR)$ with
$\|{\mu_n}\|_{ \mathcal{M}^1(1/\psi)}^{\nicefrac{1}{p}}
	\leq \|{\mu}\|_{ \mathcal{M}^1(1/\psi)}^{\nicefrac{1}{p}}$
such that $f_n =  \mathtt{H}_{\mu_n,\vv,c}^p$
converges to $f$ pointwise.
\end{lemma}
Mainly the previous Lemma \ref{lem:narrow} states that for function
$f$ that can be presented as
infinite width neural network is the pointwise limit of certain sequence
of finite width neural network $f_n$
where $f_n$ defined through sequence of measures uniformly bounded in total variation norm.

\begin{lemma}\label{lem:opteq}
Let $\psi (b) = 1+|b|^{p-1}$ where $b \in \mathbb{R}$ and $p\in \mathbb{N}$,
then we have the following
\begin{equation}\label{eq:mineq}
  \Rbar{f} = \min_{\mu\in \mathcal{M}^1(\SR),c\in \R}
  \big\|{\mu}\big\|_{ \mathcal{M}^1(1/\psi)}^{\nicefrac{1}{p}}
  ~~\text{such that}~~ f =  \mathtt{H}_{\mu,c}^p,
\end{equation}
and
\begin{equation}\label{eq:mineq2}
  \Rbarone{f} = \min_{\mu\in  \mathcal{M}^1(\SR), \vv\in \mathbb{R}^{d\times p}, c\in \R}
  \big\|{\mu}\big\|_{ \mathcal{M}^1(1/\psi)}^{\nicefrac{1}{p}}
  ~~\text{such that}~~ f = \mathtt{H}_{\mu,\vv,c}^p.
\end{equation}
\end{lemma}
\begin{proof}
We show the first equivalence \eqref{eq:mineq} for $\Rbar{f}$,
the second equivalence \eqref{eq:mineq2} for 
$\Rbarone{f}$ can be proved with similar arguments and therefore it is left to the reader.
In similar way to the proof of Lemma \ref{lem:rbareq}, we argue in two steps.

Step 1:
Let $L(f)$ be the right-hand side of  \eqref{eq:mineq} and we 
assume that $\Rbar{f}$ is finite. Then, by the equivalence of $\Rbar{f}$ given in
\eqref{eq:rbarseq0} it exists a tight sequence of measures
$\mu_n \in \mathcal{D}_\psi(\SR)$, where $\psi(b)= 1+|b|^{p-1}$ and $b\in \mathbb{R}$,
such that $\mu_n$ is uniformly bounded in total variation
norm and 
we have $ \mathtt{H}_{\mu_n}^p\rightarrow  f_0$ pointwise.
Since in $\mathcal{M}^1({\SR})$ weak convergence of measures 
is equivalent to narrow convergence \cf \cite[Theorem 6.7]{malliavin2012},
then by Prohorov's Theorem, $\{\mu_n\}$ has a subsequence
$\{\mu_{n_k}\}$ converging narrowly to a measure $\mu$,
 \cf \cite[Section 8.6 Chapter 8]{bogachev2007}.
Therefore $ f_0 =  \mathtt{H}_{\mu}^p$.
Since the  sequence of measures  $\{\mu_{n_k}\}$ converges narrowly,
we have 
$$
\|{\mu}\|_{{\mathcal{M}^1(1/\psi)}}^{\nicefrac{1}{p}} \leq
	\limsup_{k\rightarrow\infty} \|{\mu_{n_k}}\|_{{\mathcal{M}^1(1/\psi)}}^{\nicefrac{1}{p}}
	\leq \limsup_{n\rightarrow\infty} \|{\mu_{n}}\|_{{\mathcal{M}^1(1/\psi)}}^{\nicefrac{1}{p}}.
$$
Therefore
$L(f) \leq \limsup_{n\rightarrow\infty} \|{\mu_{n}}\|_{{\mathcal{M}^1(1/\psi)}}^{\nicefrac{1}{p}}$,
which implies that $L(f) \leq \Rbar{f}$ after taking the infimum over all such sequences $\{\mu_n\}$.

Step 2: In this step we assume that $L(f)$ is finite,
and let $\mu \in \mathcal{M}^1(\SR)$ such that
$ f_0 =  \mathtt{H}_\mu^p$.
In view of  Lemma \ref{lem:narrow} there exists a sequence of measures
$\mu_n \in \mathcal{D}_\psi(\SR)$
with $\|{\mu_n}\|_{\mathcal{M}^1(1/\psi)}^{\nicefrac{1}{p}} 
	\leq \|{\mu}\|_{\mathcal{M}^1(1/\psi)}^{\nicefrac{1}{p}}$,
such that $ \mathtt{H}_{\mu_n}^p\rightarrow  f_0$ pointwise.
Therefore,
$$
\Rbar{f} \leq \limsup_{n\rightarrow \infty}
	 \|{\mu_n}\|_{\mathcal{M}^1(1/\psi)}^{\nicefrac{1}{p}}
	 \leq \|{\mu}\|_{\mathcal{M}^1(1/\psi)}^{\nicefrac{1}{p}}.
$$
Seeing that $\mu$ is arbitrarily chosen in $\mathcal{M}^1(\SR)$,
with $ f_0 =  \mathtt{H}_\mu^p$,
then $\Rbar{f}\leq L(f)$.

Step1 and Step 2 clearly show that   $\Rbar{f}= L(f)$,
which conclude the claim of our lemma.
\end{proof}

\par

Next result shows that the optimization problem describing
$\Rbarone{f}$ in \eqref{eq:mineq2} reduces to \eqref{eq:evenonly}.
That is, if $f$ can be represented as an infinite width network, then
$\Rbarone{f}$ is equal to the minimal total variation norm of all even measures defining
$f$.

\begin{lemma}\label{lem:roneopt}
Let  $p\in \mathbb{N}$ be odd and $\psi(b) = 1+|b|^{p-1}$ for any $b\in \R$, then
\begin{equation}\label{eq:evenonly}
  \Rbarone{f} = \min
   \|{\mu^+}\|_{ \mathcal{M}^1(1/\psi)}^{\nicefrac{1}{p}}~~
   \text{ such that }~~ f = \mathtt{H}_{\mu^+,\vv,c}^p,
\end{equation}
where the minimum is taken over all even measure $\mu^+\in \mathcal{M}_e^1(\SR)$,
    $\vv\in\mathbb{R}^{d\times p}$ and $c\in\R$.
\end{lemma}

The proof of Lemma \ref{lem:roneopt} can be conclude in similar
way as in \cite[Lemma 7]{Ongie2019}.
For the seek of completeness we give the proof.

\begin{proof}
Let
$f = \mathtt{H}_{\mu,\vv,c}^p$ for some $\mu \in \mathcal{M}^1(\SR), \vv\in\R^{d\times p},c \in\R$.
If $\mu$ has even and odd decomposition $\mu = \mu^+ + \mu^-$
then $f = \mathtt{H}_{\mu^+,\bm 0,0}^p + \mathtt{H}_{\mu^-,\vv,c}^p = \mathtt{H}_{\mu+,\vv',c}^p$
for some $\vv'\in\Rd$. Also, by Lemma \ref{prop:evenodd},
we have $\|\mu^+\|_{\cM^1} \leq \|\mu^+ + \mu^-\|_{\cM^1} = \|\mu\|_{\cM^1}$
for any $\mu^-$ odd.
\end{proof}

\subsection{Proof of Theorem \ref{thm:main}}\label{appsec:thm1}
Let $\Sc_e(\SR)$ denote the space of \emph{even} Schwartz functions on $\SR$, 
\ie $\phi \in \Sc_e(\SR)$ if $\phi \in \Sc(\SR)$ with $\phi(w,b) = \phi(-w,-b)$ for all $(w,b)\in\SR$.
In the sequel we need the following property of the Radon transform
and its dual when acting on Schwartz function, which can be founded in \eg \cite{solmon1987}.

\par

\begin{lemma}\cite[Theorem 7.7]{solmon1987}\label{lem:Solmon}
	Let $\phi\in\Sc_e(\SR)$ and define $\varphi = \gamma_d (-\Delta)^{(d-1)/2}\Rdns{\phi}$.
	Then $\varphi\in C^\infty(\R^d)$ with $\varphi(x) = O(\|x\|^{-d})$
	and $\Delta \varphi(x) = O(\|x\|^{-d-2})$ as $\|x\| \rightarrow \infty$.
	Moreover, $\Rdn{\varphi} = \phi$.
\end{lemma}

Using the above result we show the $\cR$-norm given in Definition \ref{def:Rnorm} is well-defined:

\begin{proposition}
	Let $p\in \mathbb{N}$ be an odd number, $\kappa\in \mathbb{N}_0$ such that $\kappa<p$,
	for $f\in Lip^\kappa (\Rd)$, if the map
	${L_f(\phi) := -\frac{\gamma_d}{p!}\langle f, (-\Delta)^{(d+p)/2}\Rdns{\phi}\rangle}$
	is finite for all $\phi \in \Sc_e(\SR)$, then
	$\Rnorm{f} = \sup\prn{ L_f(\phi)}^{\nicefrac 1p}$
	is a well-defined functional taking values in $[0,+\infty]$, where the supremum 
	is taken over all $\phi \in \Sc_e(\SR)$ such that $\|\phi\|_\infty \leq 1$.
\end{proposition}

\par

\begin{proof}
	Using the definition of $\kappa$-order Lipschitz function $f$, then there exists a multi-index $\alpha$
	such that $|\partial ^\alpha f(x)| = O(\|x\|)$ where $|\alpha|=\kappa$.
	Therefore, we have $|f(x)| = O(\norm{x}^p)$, moreover 
	for any $\phi \in \Sc_e(\SR)$, by Lemma \ref{lem:Solmon}, we get
	$|(-\Delta)^{(d+p)/2}\Rdns{\phi}| = O(\|x\|^{-d-p-1})$.
	We conclude that $|f(x)(-\Delta)^{(d+p)/2}\Rdns{\phi}(x)| = O(\| x\|^{-d-1})$
	is integrable, then $\langle f, (-\Delta)^{(d+p)/2}\Rdns{\phi}\rangle$ is finite.
	If 
	$$
	\langle f, (-\Delta)^{(d+p)/2}\Rdns{\phi}\rangle\neq 0,
	$$
	we can choose the sign of $\phi$ so that $-\frac{\gamma_d}{p!}
	\langle f, (-\Delta)^{(d+p)/2}\Rdns{\phi}\rangle$ is positive,
	which implies that $\Rnorm{f} \geq 0$.
\end{proof}

The following lemma analyses the case where using an even measure belongs to $\cM_e(\SR)$
instead of a measure with density in the Schwatrz class can
show that equality \eqref{eq:identity_laplacian_radon} holds true also in the distributional case.

\begin{lemma}\label{lem:Lap2Rdn}
	Let $p\in \mathbb{N}$ be an odd number,
	$f = \mathtt{H}_{\mu,\vv,c}^p$ for any $\mu \in \cM_e(\SR), \vv\in \R^{d\times p}, c\in\R$.
	Then  we have
	$\langle f, \Delta^{\frac{p+1}{2}} \varphi\rangle = \langle \mu, p!\frac{\Rdn{\varphi}}\psi\rangle$
	 for all $\varphi \in C^\infty(\Rd)$ such that $\varphi(x) = O(\|x\|^{-d})$
	 and $\Delta^{\frac{p+1}{2}} \varphi(x) = O(\|x\|^{-d-p-1})$ as $\|x\|\rightarrow \infty$
	 and $\psi(b)=1+|b|^{p-1}$ where $b\in \R$.
\end{lemma}

\begin{proof}
	
	We start by setting, for odd $p \in \mathbb{N}$,
	$f = \mathtt{H}_{\mu, \vv, c}^p$ with $\mu \in \cM_e(\SR)$, $\vv\in \R^{d\times  p}$, $c\in \R$.
	The number $\frac{p+1}{2}$ is even, and since polynomials up to order $p$
	vanish under the application of $\frac{p+1}{2}$ Laplacian, then
	$\langle f, \Delta^{\frac{p+1}{2}} \varphi\rangle
		= \langle  \mathtt{H}_{\mu}^p,\Delta^{\frac{p+1}{2}} \varphi\rangle$,
	for all $\varphi \in C^\infty(\R^d)$ with $\varphi(x) = O(\|x\|^{-d})$
	and $\Delta^{\frac{p+1}{2}} \varphi(x) = O(\|x\|^{-d-p-1})$ when $\|x\|\rightarrow \infty$.
	Hence the problem is reduced to the case where $f = \mathtt{H}_{\mu}^p$, and 
	$$
	\int_{\R^d} f(x) \Delta^{\frac{p+1}{2}} \varphi(x) \, dx
		 = \int_{\R^d} 
		 \underbrace{\int_{\SR}\frac{1}{2}\frac{|w x - b|^p-|b|^p}{1+|b|^{p-1}}
		 \,d \mu(w,b)}_{ \mathtt{H}_{\mu^+}^p(x)}
		 \Delta^{\frac{p+1}{2}} \varphi(x) \, dx.
	$$
	Last equality holds thanks to the fact that for any $t\in \R$
	and any odd $p$ $\relu{t}^p  = \half(|t|^p+ t^p)$  and that 
	$\mathtt{H}_\mu^p = \mathtt{H}_{\mu^+}^p + \mathtt{H}_{\mu^-}^p$, where
	$\mathtt{H}_{\mu^-}^p$ is a polynomial in $x$ of order $p$ that can be neglected in this situation
	since it vanishes under the application of $\frac{p+1}{2}$ Laplacian.
	Moreover, we have
	$$
	\int_{\R^d} f(x) \Delta^{\frac{p+1}{2}} \varphi(x) \, dx
		 = \int_{\SR} \left(
		 \int_{\R^d}\frac{1}{2}\frac{|w x - b|^p-|b|^p}{1+|b|^{p-1}}
		 \Delta^{\frac{p+1}{2}} \varphi(x)\,dx \right) d\mu(w,b)
	$$
	where we applied Fubini's theorem to exchange the order of integration,
	which is justified by
	\begin{equation*}
	      \frac{1}{2}\int_{\SR}\frac{|w x - b|^p-|b|^p}{1+|b|^{p-1}}\,d |\mu|(w,b) \leq \|\mu\|_{\mathcal{M}^1}
	      	\sum_{k=1}^p\binom{p}{k}\|x\|^k
	\end{equation*} 
	and by the fact that $\Delta^{\frac{p+1}{2}}\varphi(x) = O(\|x\|^{-d-p-1})$,
	as it is mentioned in the statement.
	Then, we get
	 $$
	 \frac{1}{2}\int_{\SR}\frac{|w x - b|^p-|b|^p}{1+|b|^{p-1}}\,d |\mu|(w,b)
	 	|\Delta^{\frac{p+1}{2}}\varphi(x)| = O(\|x\|)^{-d-1},
	 $$
	which leads to
	$$
	\int_{\Rd}  \frac{1}{2}\int_{\SR}\frac{|w x - b|^p-|b|^p}{1+|b|^{p-1}}\,d |\mu|(w,b)
		|\Delta^{\frac{p+1}{2}}\varphi(x)|\,dx < \infty.
	$$
	
	Let $r_{w,b}(x) :=\frac{|w x - b|^p-|b|^p}{1+|b|^{p-1}}$ for any $(w, b)\in \SR$
	and $x\in \Rd$. In the distributions sense we can show that
	$\Delta ^{\frac{p+1}{2}}r_{w,b}(x) = p!\frac{\delta(w x - b)}{1+|b|^{p-1}}
		=  p!\frac{\delta(w x - b)}{\psi(b)}$.
	Namely, for any test function $\varphi \in \Sc(\Rd)$ we have the following identity
	\begin{equation}
	         \int_{\Rd} r_{ w,b}(x) \Delta^{\frac{p+1}{2}} \varphi(x)\, dx 
	         = \frac{p!}{\psi(b)}\int_{w x = b} \varphi(x) \, ds(x) = \frac{p!}{\psi(b)}\Rdn{\varphi}(w,b).
	\end{equation}
	The Radon transform $\Rdn{\varphi}$ is well-defined for
	smooth functions that decay as $O(\|x\|^{-d})$.
	Hence, by continuity $\Delta^{\frac{p+1}{2}} r_{w,b}(x)$
	extends uniquely to a distribution acting on $C^\infty$ functions decay like $O(\|x\|^{-d})$.
	In view of the previous calculus, we get
	\begin{align}
	    \int_{\R^d} f(x) \Delta^{\frac{p+1}{2}} \varphi(x) \, dx
	        & = \int_{\SR} \left(\int_{\R^d}r_{w,b}(x) \Delta^{\frac{p+1}{2}} \varphi(x) \, dx\right) d\mu(w,b)
	        \\
	        & = p!\int_{\SR} \frac1{\psi(b)}\Rdn{\varphi}(w,b)\, d\mu(w,b),
	\end{align}
	which shows the claimed statement in the lemma.
\end{proof}

The following lemma shows $\Rnorm{f}$ is finite if and only if
$f$ is an infinite width net, in which case $\Rnorm{f}$
is given by the total variation norm of the unique even measure defining $f$.

\begin{lemma}\label{lem:main}
	Let $p\in \mathbb{N}$ be an odd number, $\kappa\in \mathbb{N}_0$ such that $\kappa<p$,
	for $f\in Lip^\kappa (\Rd)$.
	Then $\Rnorm{f}$ is finite
	if and only if there exists a unique even measure $\mu \in \cM_e(\SR)$,
	 $\vv\in\R^{d\times p}$ and $c\in\R$ with
	$f = \mathtt{H}_{\mu,\vv,c}^p$, such that
	$\Rnorm{f} =\| \mu\|_{\mathcal{M}^1(p!/\psi)}^{\nicefrac 1p}$. 
\end{lemma}

\begin{proof}
We start by the direct sense, that is, we assume that $\Rnorm{f}$ is finite.
In view of Definition \ref{def:Rnorm} there exits $\kappa<p$ such that the function $f$ is a
$\kappa$-order Lipschitz function. Moreover the linear functional
$$
	L_f(\phi) = -(-1)^{(p+1)/2} \gamma_d\langle f, (-\Delta)^{(d+p)/2}\Rdns{\phi}\rangle,
	\quad \phi \in \Sc_e(\SR)
$$
is continuous on $\Sc_e(\SR)$ with norm $\Rnorm{f}^p$.
Using an extension argument, and the fact that 
$\Sc_e(\SR)$ is a dense subspace of $C_{0,e}(\SR)$,
there exists a unique extension $\tilde{L}_f$ to all of $C_{0,e}(\SR)$
under the same norm. Riesz representation theorem declares that $\tilde{L}_f$
can be seen as integration against a measure, that is,
there is a unique even measure $\nu \in \cM_e(\SR)$
such that the extension $\tilde{L}_f(\phi) = \int \phi\, d\nu$ for all $\phi \in C_{0,e}(\SR)$
and $\Rnorm{f} = \|\nu\|_{\mathcal{M}^1}^{\nicefrac 1p}$. 

Without loss of generality, we can write
\footnote{This holds true since any $\phi \in C_{0,e}(\SR)$
can be written as $\frac{p!}{\psi(b)}\Phi(w, b)$
where $\Phi(w, b) = \frac{\psi(b)}{p!}\phi(w, b)\in  C_{0,e}(\SR)$.}
$$
\tilde{L}_f(\Phi) = p!\int \Phi\, \frac{d\mu}{\psi}
\quad \text{ for all } \Phi \in C_{0,e}(\SR), \,\psi(b) = 1+|b|^{p-1},
$$
where $\mu \in \cM_e(\SR)$, hence $\Rnorm{f} =  \|p!\mu\|_{\mathcal{M}^1(1/\psi)}^{\nicefrac 1p}$.

\par

In this step it remains to show that $f$ can be expressed as $\mathtt{H}_{\mu, \vv, c}^p$,
where $\mu \in \cM_e(\SR)$, $\vv\in\R^{d\times p}$ and $c\in\R$.
We argue in the sense of tempered distribution.Namely, we show that
$\scal{f}{\varphi} = \scal{\mathtt{H}_{\mu, \vv, c}^p}{\varphi}$ for any $\varphi \in \Sc(\Rd)$.
In view of the discussion in the proof of Lemma \ref{lem:Lap2Rdn},
for any test function $\varphi \in \Sc(\R^d)$, it holds that
$\scal{ \Delta^{(p+1)/2} \mathtt{H}_{\mu}^p}{\varphi} =
	p! \scal{\mu}{\frac{\Rdn{\varphi}}{\psi}}$.
Since for any $\varphi \in \Sc(\SR)$ the Radon transform $\Rdn{\varphi}$
belongs to $\Sc_e(\SR)$ \cf \cite[Theorem 2.4]{helgason1999radon}.
Hence   
$ \tilde{L}_f(\Rdn{\varphi})$ can be written as $p!\scal{\mu}{\frac{\Rdn{\varphi}}{\psi}}$
which affirms that
$\scal{\Delta^{(p+1)/2} \mathtt{H}_{\mu}^p}{\varphi}  = L_f(\Rdn{\varphi})$.
Seeing that $L_f(\varphi)$ can be written as
\begin{align*}
    -(-1)^{(p+1)/2}\gamma_d \scal{f}{(-\Delta)^{(d+p)/2}\Rdns{\Rdn{\varphi}}}
\intertext{which equals to}
-\gamma_d\scal{f}{\Delta^{(p+1)/2}(-\Delta)^{(d-1)/2}\Rdns{\Rdn{\varphi}}}.
\end{align*}
Moreover, using the inversion formula for Radon transform, \ie
$-\gamma_d (-\Delta)^{(d-1)/2} \Rdns{\Rdn{\varphi}} = \varphi$ for all $\varphi \in \Sc(\Rd)$,
it follows that $L_f(\varphi) = \scal{f}{\Delta^{(p+1)/2}\varphi} = \scal{\Delta^{(p+1)/2}f}{\varphi}$.
Hence $\scal{\Delta^{(p+1)/2} \mathtt{H}_{\mu}^p}{\varphi}  =
\scal{\Delta^{(p+1)/2}f}{\varphi}$ for any $\varphi \in \Sc(\Rd)$.

The previous equality yields that $\Delta^{(p+1)/2}f = \Delta^{(p+1)/2} \mathtt{H}_{\mu}^p$
holds true in the sense of tempered distributions.
Therefore $\Delta^{(p+1)/2}\prn{f - \mathtt{H}_{\mu}^p}=0$,
hence $f - \mathtt{H}_{\mu}^p = \mathscr{P}$ where $\mathscr{P}$ is a polynomial satisfies
the following $\Delta^{(p+1)/2} \mathscr{P}(x)= 0$  for any $x\in\R^d$.
Observing that $f$ and $\mathtt{H}_{\mu}^p$ are $\kappa$-order Lipschitz, means that
$\mathscr{P}$ is a polynomial in $x$ at most of order $p$. Therefore, there exists
$\vv = (v_1, \dots, v_p)\in \R^{d\times p}$ 
and $c\in \R$ such that $\mathscr{P}(x) = \sum_{k=1}^{p}v_k x^k + c$.
Consequently, the previous argument concludes that $f = \mathtt{H}_{\mu, \vv, c}^p$.

\par

We now prove the inverse implication: we assume that for some
$\mu \in \cM_e(\SR),\vv\in\R^{d\times p},c\in\R$,  $f$ can be expressed as
$\mathtt{H}_{\mu,\vv,c}^p$.
Given an even test function $\phi \in \Sc_e(\SR)$ and let
$$
\varphi = -(-1)^{(p+1)/2}\gamma_d (-\Delta)^{(d-1)/2}\Rdns{\phi}.
$$
Then the function $\varphi$ satisfies the claims in Lemma \ref{lem:Solmon},
hence  $\varphi\in C^\infty(\R^d)$ with $\varphi(x) = O(\|x\|^{-d})$,
$\Delta \varphi(x) = O(\|x\|^{-d-2})$ as $\|x\| \rightarrow \infty$
and $\phi = \Rdn{\varphi}$.
Moreover, by  Lemma \ref{lem:Lap2Rdn},
the linear functional $L_f(\phi)$ can be characterized as follows
\begin{equation}\label{eq:Lf_charact}
    L_f(\phi) = \langle f, \Delta^{(p+1)/2} \varphi \rangle
     = \langle \mu, p!\frac{\Rdn{\varphi}}{\psi} \rangle
     = \langle \mu, p!\frac\phi\psi \rangle.
\end{equation}
Using  Lemma \ref{lem:Lap2Rdn}, the characterization of the linear functional
$L_f(\phi)$ given in \eqref{eq:Lf_charact}
and the fact that $\Sc(\R^d)$ is a dense subspace of $C_0(\R^d)$ we get
\begin{align*}
	\Rnorm{f}  &= \sup \{\prn{\scal{\mu}{p!\frac\phi\psi}}^{\nicefrac 1p}
		: \phi \in \Sc_e(\SR), \|\phi\|_\infty \leq 1\} \\
	& = \sup \{\prn{\scal{\mu}{p!\frac\phi\psi}}^{\nicefrac 1p} : \phi \in C_{0, e}(\SR), 
		\|\phi\|_\infty \leq 1\}
\end{align*}
then by the dual characterization of the total variation norm we conclude that 
$\Rnorm{f} = \|\mu\|_{\mathcal{M}^1(p!/\psi)}^{\nicefrac 1p}$.
Uniqueness can be achieved by assuming the existence of 
$\mu, \nu \in \cM_e(\SR)$, $\vv, \vv' \in \R^{d\times p}$, $c ,c'\in\R$ such that
$\mathtt{H}_{\mu,\vv,c}^p = \mathtt{H}_{\nu,\vv',c'}^p$.
Since the difference
$\mathtt{H}_{\mu,\vv,c}^p - \mathtt{H}_{\nu,\vv',c'}^p=\mathtt{H}_{\mu- \nu,\vv-\vv',c- c'}^p$
hence by the previous arguments 
$$
\Rnorm{\mathtt{H}_{\mu- \nu,\vv-\vv',c- c'}^p}
	= \norm{{\mu- \nu}}_{\mathcal{M}^1(p!/\psi)} ^{\nicefrac 1p}= 0.
$$
Then $\mu=\nu$ which implies that $\vv' = \vv$ and $c = c'$ and so the uniqueness.
\end{proof}

\par

Lemma \ref{lem:uniqueness} is an immediate consequence of the characterization
in Lemma \ref{lem:roneopt} and the uniqueness result given in Lemma \ref{lem:main}.
In the case where $p=1$ a similar result was proved by Ongie et al.  \cf \cite{Ongie2019}
for the sake of completeness we give the proof.

\begin{proof}[Proof of Lemma \ref{lem:uniqueness}]
We treat the case where $\Rbarone{f}$ is finite.
By Lemma \ref{lem:roneopt}, we can characterize the $\cR$-norm for a given function $f$
expressed by a neural network $\mathtt{H}_{\mu, \vv, c}^p$ where $\mu \in \cM_e^1(\SR)$
is an even measure, $\vv\in\R^{d\times p}$, $c\in\Rd$,
as follows: $\Rbarone{f}$ is the minimum of $\| {\mu^+}\|_{\mathcal{M}^1(1/\psi)}^{\nicefrac 1p}$
over all even measures $\mu^+ \in \cM_e^1(\SR)$,  $\vv'\in\R^{d\times p}$ and $c'\in\R$
such that $f = \mathtt{H}_{\mu^+, \vv', c'}^p$.
The uniqueness can be justified in a similar way to the proof of Lemma \ref{lem:main}, hence
there is a unique even measure $\mu^+ \in \cM_e^1(\SR)$, $\vv\in\R^{d\times p}$,
and $c\in\R$ such that $f = \mathtt{H}_{\mu^+,\vv,c}^p$. 
To sum up, $\Rbarone{f} = \|{\mu^+}\|_{\cM^1(1/\psi)}^{\nicefrac 1p}$ which conclude the lemma. 
\end{proof}

\par

\begin{proof}[Proof of Theorem \ref{thm:main}]
Assume that $\Rbarone{f}$ is finite.
Thanks to Lemma \ref{lem:uniqueness}, there exists a unique even measure $\mu\in \cM_e(\SR)$
such that $\Rbarone{f} =  \|{\mu}\|_{\cM^1(1/\psi)}^{\nicefrac 1p}$ and
$f = \mathtt{H}_{\mu,\vv,c}^p$ for some $\vv\in\R^{d\times p},c\in\R$.
In view of Lemma \ref{lem:main} we have  $\Rnorm{f} = \|{\mu}\|_{\cM^1(p!/\psi)}^{\nicefrac1p}$
Consequently, $\prn{p!}^{\nicefrac 1p}\Rbarone{f} = \Rnorm{f}$.
In reverse, let $\Rnorm{f}$ be finite, then by Lemma \ref{lem:main}
$f = \mathtt{H}_{\mu,\vv,c}^p$ for a unique even measure $\mu \in \cM_e(\SR)$,
$\vv\in\R^{d\times p}$ and $c\in\R$,
with $\Rnorm{f} =  \|{\mu}\|_{\cM^1(p!/\psi)}^{\nicefrac 1p}$.
Moreover using Lemma \ref{lem:uniqueness},
$ \prn{p!}^{\nicefrac 1p}\Rbarone{f} =  \|{\mu}\|_{\cM^1({p!}/\psi)}^{\nicefrac1p} = \Rnorm{f}$.
\end{proof}

\end{document}